\documentclass[nohyperref,dvipsnames]{article}

\usepackage{microtype}
\usepackage{graphicx}
\usepackage{subfigure}
\usepackage{booktabs} 

\usepackage{hyperref}

 \usepackage[accepted]{icml2022}

\usepackage{amsmath}
\usepackage{amssymb}
\usepackage{mathtools}
\usepackage{amsthm}
\usepackage{amsfonts}
\usepackage{stmaryrd}

\usepackage[capitalize,noabbrev]{cleveref}

\usepackage{csquotes}
\usepackage{xcolor}
\usepackage{afterpage}
\usepackage{tikz,tikz-3dplot}
\usepackage{amsmath,amsthm}
\usepackage{amssymb}
\usepackage{amsfonts}
\usepackage{times}
\usetikzlibrary{positioning,arrows.meta,quotes}
\usetikzlibrary{shapes,snakes}
\usetikzlibrary{bayesnet}
\tikzset{>=latex}

\usepackage[shortlabels]{enumitem}
\usepackage{mathdots}
\usepackage{bm}

\theoremstyle{plain}
\newtheorem{theorem}{Theorem}[section]

\newtheorem{lemma}[theorem]{Lemma}

\theoremstyle{definition}
\newtheorem{definition}[theorem]{Definition}

\theoremstyle{remark}

\usepackage[textsize=tiny]{todonotes}

\icmltitlerunning{Implicit Regularization with Polynomial Growth in Deep Tensor Factorization}

\begin{document}

\twocolumn[
\icmltitle{Implicit Regularization with Polynomial Growth in Deep Tensor Factorization}



\icmlsetsymbol{equal}{*}

\begin{icmlauthorlist}
\icmlauthor{Kais Hariz}{amu,enit}
\icmlauthor{Hachem Kadri}{amu}
\icmlauthor{Stéphane Ayache}{amu}
\icmlauthor{Maher Moakher}{enit}
\icmlauthor{Thierry Artières}{amu,centrale}
\end{icmlauthorlist}

\icmlaffiliation{enit}{LAMSIN, National Engineering School of Tunis, University of Tunis El Manar, Tunis, Tunisia}
\icmlaffiliation{amu}{Aix Marseille University, CNRS, LIS, Marseille, France}
\icmlaffiliation{centrale}{Ecole Centrale de Marseille, Marseille, France}

\icmlcorrespondingauthor{Kais Hariz}{kais.hariz@univ-amu.fr}
\icmlcorrespondingauthor{Hachem Kadri}{hachem.kadri@univ-amu.fr}

\icmlkeywords{Implicit Regularization, Deep Learning, Tensor Factorization}

\vskip 0.3in
]



\printAffiliationsAndNotice{}  

\begin{abstract}
We study the implicit regularization effects of deep learning in tensor factorization. While implicit regularization in deep matrix and \enquote*{shallow} tensor  factorization via linear and certain type of non-linear neural networks promotes low-rank solutions with at most quadratic growth, we show that its effect in deep tensor factorization grows polynomially with the depth of the network. This provides a remarkably faithful description of the observed experimental behaviour. Using numerical experiments, we demonstrate the benefits of this implicit regularization in yielding a more accurate estimation and better convergence properties. 
\end{abstract}

\section{Introduction}
\label{sec:intro}

A major challenge in deep learning is to understand the underlying mechanisms behind the ability of deep neural networks to generalize. 
This is of fundamental importance to reconcile the observation that deep neural networks generalize well even for situations where the number of learnable parameters is much larger than the number of training data.
Starting with the report by~\citep{neyshabur2014search}, a body of work has emerged exploring the role of implicit regularization in deep learning~\citep{gunasekarimplicit, arora2019implicit, kumar2020implicit, razin2020implicit, li2021towards, razin2021implicit, milanesi2021implicit, zou2021benefits}.
Our work contributes to this effort by providing insight into the behaviour of implicit regularization in deep tensor factorization where we focus on deep versions of canonical rank and Canonical-Polyadic (CP) factorizations~\cite{Kolda}.

Attempts of studying implicit regularization in deep learning have identified matrix completion as a suitable test-bed~\citep{arora2019implicit}. 
\citet{gunasekarimplicit} observed that  
for matrix factorization when there are no constraints on the rank, the solution of  the optimization problem via gradient descent turns out to be a low-rank matrix. Furthermore, they conjectured that, with small enough learning rate and initialization, gradient descent on full-dimensional matrix factorization converges to the solution with minimal nuclear norm.
\citet{arora2019implicit} and \citet{razin2020implicit} extended the analysis to deep matrix factorization and showed in this case that implicit regularization of gradient descent cannot be formulated as a norm-minimization problem. By studying the dynamics of  gradient descent, they found theoretically and experimentally that it instead promotes sparsity of the singular values of the learned matrix, indicating that implicit regularization in deep learning has to be studied  from a dynamical point of view. 
Moreover, \citet{razin2021implicit} studied implicit regularization in \enquote*{shallow} tensor decomposition and showed an equivalence between a tensor completion task and a prediction problem with a nonlinear neural network, stressing the interest of studying the tensor completion task. 

Our main contributions focus on implicit regularization in {deep} Canonical-Polyadic tensor factorization and can be summarized as follows:
\begin{itemize}
    \item we prove that the effect of the implicit regularization in deep tensor CP factorization via gradient descent grows polynomially with the depth of the factorization~(Theorem~\ref{th:dy_deepcp}),
	\item we theoretically show under some conditions that this effect in the overparameterized regime leads to produce solutions with low tensor rank~(Theorem~\ref{th:rank_deepcp}),
	\item we perform numerical experiments that support our theoretical results, illustrating that the implicit regularization could yield more accurate estimations and better convergence properties~(Section~\ref{sec:xp}).
\end{itemize}

\section{Problem Setup and Background}

\begin{figure*}[th]
	\centering
	\includegraphics[scale=0.5]{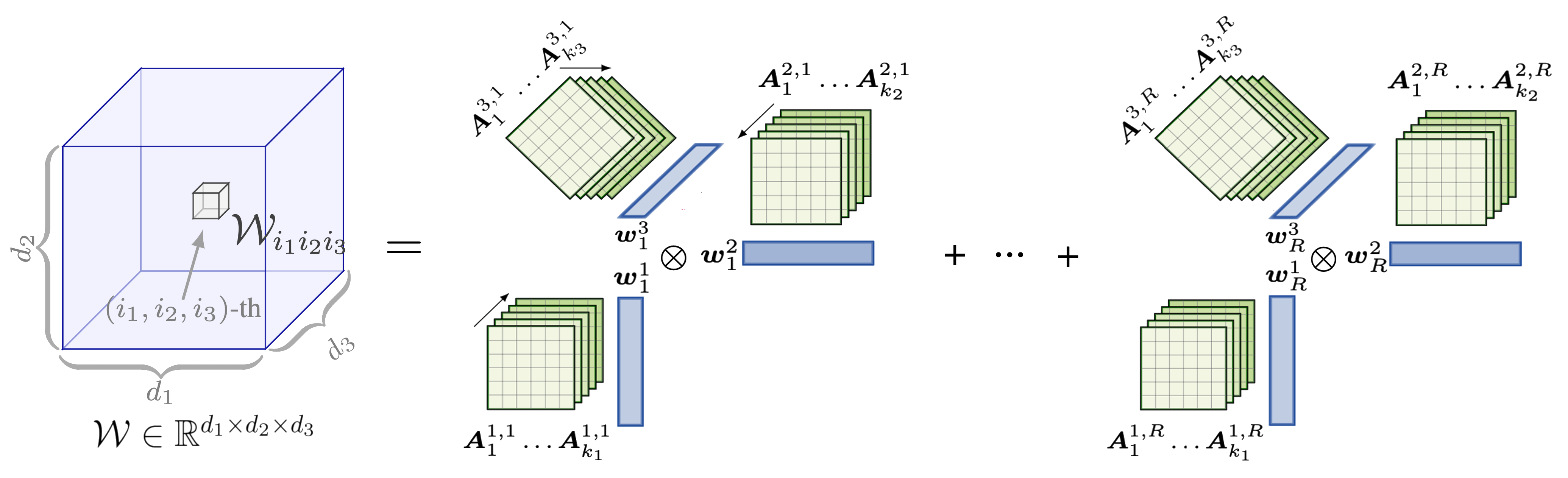}
	\caption{Overparameterized deep CP factorization.}
		\label{fig:deepcp}
\end{figure*}

We study implicit regularization in deep tensor factorization, and so we consider the problem of tensor completion  with {overparameterized} factorization. By \enquote*{overparameterized} we mean that no
assumptions are made on the rank of the tensor. This is crucial in order to analyze the effect of the implicit regularization on the learned tensor.
\paragraph{Notation and terminology.}
Throughout the paper, boldface lowercase letters such as $\bm{w}$ denote vectors, boldface capital letters such as $\bm{A}$, $\bm{W}$ denote matrices, and calligraphic letters such as $\mathcal{W}$ denote tensors.
The $(i_1, i_2, \ldots, i_N)$-th entry of an $N$-order tensor  $\mathcal{W}\in\mathbb{R}^{d_1\times d_2\times \cdots\times d_N}$ will be denoted by $\mathcal{W}_{i_1, i_2, \ldots, i_N}$ where $i_n=1, 2, \ldots, d_n$, for all $n=[1,\ldots, N]$. Given two tensors $\mathcal{V}, \mathcal{W}\in \mathbb{R}^{d_1\times d_2\times \cdots\times d_N}$ their scalar product writes
\begin{equation*}
	\langle\mathcal{V}, \mathcal{W} \rangle = \sum_{i_1}\sum_{i_2}\cdots \sum_{i_N}\mathcal{V}_{i_1, i_2\cdots i_N} \mathcal{W}_{i_1, i_2\cdots i_N},
\end{equation*}
and the Frobenius norm of $\mathcal{W}$ is defined as $\|\mathcal{W}\|:=\sqrt{\langle\mathcal{W}, \mathcal{W}\rangle}$.
We also write $\|\bm{A}\|$ and $\|\bm{w}\|$ for the Frobenius norm of a matrix $\bm{A}$ and the Euclidean norm of a vector $\bm{w}$, respectively.

Given $N$ vectors $\bm{w}_1\in\mathbb{R}^{d_1}, \bm{w}_2\in\mathbb{R}^{d_2}, \ldots, \bm{w}_N\in\mathbb{R}^{d_N} $, their outer product is the tensor whose $(i_1, \ldots, i_N)$-th entry writes $\left(\bm{w}_1\otimes \bm{w}_2\otimes\cdots\otimes \bm{w}_N\right)_{i_1, i_2\ldots, i_N}=(\bm{w}_1)_{i_1}(\bm{w}_2)_{i_2}\ldots (\bm{w}_N)_{i_N}$ for all $i_n\in[1,\ldots, d_n]$, $n\in[1, \ldots, N]$. An $N$-th order tensor $\mathcal{W}$ is called a rank-$1$ tensor if it can be written as the outer product of $N$ vectors, i.e. $\mathcal{W}=\bm{w}_1\otimes \bm{w}_2\otimes\cdots\bigotimes \bm{w}_N$. 
This leads to the notion of canonical rank and Canonical-Polyadic (CP) factorization~\cite{Kolda}.

The canonical rank of an arbitrary $N$-th order tensor $\mathcal{W}$ is the minimal number of rank-$1$ tensors
 that sum up to $\mathcal{W}$. 
 Decompositions into rank-1 terms are called CP factorization. A rank R tensor $\mathcal{W}$ can be written as:\\
 \begin{equation}
 	\label{eq:cp}
 	\mathcal{W} = \sum_{r=1}^R \bm{w}_r^1 \otimes\ldots\otimes \bm{w}_r^N.
 \end{equation}
We will use the term \enquote*{block} to refer to a rank-1 tensor of the CP decomposition so that $\mathcal{W}$ in Eq.~\eqref{eq:cp} has $R$ blocks and $\bm{w}_r^1 \otimes\ldots\otimes \bm{w}_r^N$ is its $r$-th block.
In this work we focus only on the CP factorization of incomplete tensor. 
We did not consider any other tensor decomposition, such as Tucker and TensorTrain~\cite{Kolda, grasedyck2013literature}, which remain then out of the scope of this paper.

\paragraph{Overparameterized deep CP factorization.}
Let $\mathcal{A}\in\mathbb{R}^{d_1 \times\cdots\times d_N}$ the ground truth tensor we want to recover and let us denote by $\Omega\subset\{1, 2, \ldots d_1\}\times \cdots \times\{1, 2, \ldots, d_N\}$ the set which indexes the observed entries. 
To tackle the problem of tensor completion~\cite{gandy2011tensor, song2019tensor}, we minimize the reconstruction loss defined by
\begin{equation*}\label{tensor_completion_loss}
	\mathcal{L}(\mathcal{W}) = \frac{1}{|\Omega|}\sum_{(i_1, \ldots, i_N)\in\Omega} \ell\left(\mathcal{W}_{i_1, \ldots, i_N} - \mathcal{A}_{i_1, \ldots, i_N}\right),
\end{equation*}
where $\ell$ is differentiable and locally smooth. The square loss, which we used
in our experiments, is obtained when $\ell(z) = \frac{1}{2} z^2, \forall z\in\mathbb{R}$.
In order to be able to investigate the mechanisms of implicit regularization in deep tensor factorization, we  consider the following overparameterized deep CP decomposition (see Figure~\ref{fig:deepcp}):
 \begin{equation}
 	\label{eq:deepcp0}
 		\mathcal{W} \!=\! \sum_{r=1}^R \! \left(\!\bm{A}_1^{1,r}\!\!\ldots\! \bm{A}_{k_1}^{1,r} \bm{w}_r^1\!\right) \otimes\!\ldots\!\otimes \left(\!\bm{A}_1^{N,r}\!\!\ldots\! \bm{A}_{k_N}^{N,r}\bm{w}_r^N\!\right)\!\!,
\end{equation}
where $\bm{w}_r^n \in \mathbb{R}^{d_n}$ and $\bm{A}_i^{n,r} \in \mathbb{R}^{d_n \times d_n}$, $\forall n=1,\ldots,N$ and $i=1,\ldots, k_n$.
We take a large $R$ value to avoid any restriction of the CP rank. 
The matrices $\bm{A}_1^{n,r},\ldots, \bm{A}_{k_n}^{n,r}$ can be seen as a deep matrix factorization for the $n$-th mode and the $r$-th block of the CP decomposition and $k_n$ is the depth of the factorization on the mode $n$. All these matrices are of dimension $d_n \times d_n$ such that no constraint on the rank is imposed~\citep{arora2019implicit}. When they are fixed to the identity matrix, we recover the standard CP decomposition as defined in Eq.~\eqref{eq:cp}.
In the sequel, we use  the following compact form to rewrite  Eq.~\eqref{eq:deepcp0}:
 \begin{equation}
	\label{eq:deepcp}
	\mathcal{W} = \sum_{r=1}^R  \bigotimes_{n=1}^N  \prod_{i=1}^{k_n}  \bm{A}_i^{n,r} \bm{w}_r^n.
\end{equation}
Our main aim is to highlight the role of implicit regularization in {deep CP tensor factorization} and characterize its dependence on the depth of the factorization.
%
In the following, we consider learning a tensor $\mathcal{W}$ which has the form~(\ref{eq:deepcp})
by minimizing the loss function $\mathcal{L}(\mathcal{W}) = \Phi\Big(\{\bm{w}_r^n\}_{r=1\ n=1}^{\ R\ \ \ N}, \{\bm{A}_i^{n,r}\}_{r=1\ n=1\ i=1}^{\ R\ \ \ N\ \ k_n}\Big )$ using gradient descent.
With infinitesimally small learning rate and non zero initialization, we have 
$$\frac{d}{dt}\bm{w}_r^n(t)=-\frac{\partial}{\partial \bm{w}_r^n}\Phi \Big(\{\bm{w}_{r'}^{n'}\!(t)\}_{r',n'},\{\bm{A}_i^{n',r'}\!(t)\}_{r',n',i}\Big),$$
and
$$\frac{d}{dt}\bm{A}_i^{n,r}(t)=\!-\frac{\partial}{\partial\bm{A}_{i}^{n,r}}\Phi \Big(\!\{\bm{w}_{r'}^{n'}\!(t)\}_{r',n'},\! \{\bm{A}_i^{n',r'}\!(t)\}_{r',n',i}\Big).$$


Note that \citet{razin2021implicit} have made a connection between  tensor completion via CP tensor factorization and a certain-type of  non-linear {one hidden layer} neural network, motivating their work as a important step towards the study of implicit regularization in standard neural networks. From this perspective, our overparameterized CP factorization may be viewed as an extension of this statement to the deep setting.

\paragraph{Related work.}
The works that are most related to ours are~\citet{arora2019implicit,razin2021implicit}.
\citet{arora2019implicit} considered deep matrix factorization, which consists in parameterizing the learned matrix $\bm{W}\in\mathbb{R}^{d_1\times d_2}$  as $\bm{W}=\bm{W}_k \bm{W}_{k-1}\ldots \bm{W}_2 \bm{W}_1$ for some $k\in\mathbb{N}$ and with $\{\bm{W}_i\}_{i=1}^k$ to be such that no constraint on the rank is present. Notice that deep matrix factorization is a generalization of the shallow matrix factorization setup investigated by~\cite{gunasekarimplicit}, which corresponds to the case where $k=2$. 
They observed that depth enhances recovery performances. This led them to study the dynamics in optimization and they found out that gradient descent promotes sparsity of singular values of $\bm{W}$, as summarized in the theorem below.
\begin{theorem}[\citealp{arora2019implicit}]
	\label{th:dy_deepm}
	For depth $k\geq 2$, for any $r=1,\ldots,\min(d_1,d_2)$,
	\begin{equation}
		\label{eq:dy_deepm}
	\frac{d}{dt} \sigma_r(t) = k \alpha_r(t) \cdot \left(\sigma_r(t)\right)^{2-\frac{2}{k}},
	\end{equation}
where $\alpha_r(t) = \left\langle - \nabla \mathcal{L}(\bm{W}(t)), \bm{u}_r(t) \bm{v}_r^\top(t) \right\rangle$, $\sigma_r(t)$ is the $r$-th singular value of $\bm{W}$ at time $t\geq 0$, and $\bm{u}_r(t)$ and $\bm{v}_r(t)$ are its $r$-th singular vectors.
\end{theorem}
\citet{razin2021implicit} extended this analysis to CP tensor factorization (see Eq.~\ref{eq:cp}) and showed the following result.
\begin{theorem}[\citealp{razin2021implicit}]
	\label{th:dy_cp}
	Under certain assumptions, for any $r = 1,\ldots,R$,
	\begin{equation}
		\label{eq:dy_cp}
		\frac{d}{dt}\left\|\bigotimes_{n=1}^N  \bm{w}_r^n(t)\right\|=N \gamma_r(t) \cdot \left\|\bigotimes_{n=1}^N \bm{w}_r^n(t)\right\|^{2-\frac{2}{N}},  
	\end{equation}
where 
$\gamma_r(t) = \left\langle -\nabla \mathcal{L}(\mathcal{W}(t)),\underset{n=1}{\overset{N}\bigotimes}  \frac{\bm{w}_r^{n}(t)}{\|\bm{w}_r^{n}(t)\|}\right\rangle$, $N$ is the order of the tensor $\mathcal{W}$ to be learned, and $\bm{w}_r^n$ is the vector  of the $n$-th mode and the $r$-th block of the CP decomposition.
\end{theorem}
This shows that training a CP tensor factorization
via gradient descent with small learning rate and near-zero
initialization tends to produce tensors with low canonical rank.
Note that the result in Eq.~\eqref{eq:dy_deepm} depends on the depth of the factorization $k$, while the one in Eq.~\eqref{eq:dy_cp} depends on the order of the tensor $N$. In both cases the impact of implicit regularization \textit{grows} at most \textit{quadratically}, with either  $k$ or $N$. 

As far as we are aware, the only work on implicit regularization in deep tensor factorization appears in~\citet{milanesi2021implicit}, in which the authors considered Tucker and Tensor Train (TT) decompositions and observed that, even in the case where the rank is not constrained, only a small number of higher-order singular values~\citep{de2000multilinear} and TT singular values~\citep{oseledets2011tensor} are retained by a gradient-based neural network.\footnote{After this paper was submitted, a paper by \citet{razin2022implicit} was released on arXiv, which studied implicit regularization in hierarchical tensor factorization.} However, no
theoretical justification is given there.

\section{Main Results}
\label{sec:main_results}
We  now present our main results, to be proved in Section~\ref{sec:proofs}.
Following the idea of~\citet{razin2021implicit}, we provide a dynamical characterization of the trajectories of the norm of each block of the deep CP factorization.
To proceed we need the following definition.
\begin{definition}
\label{def:doubly}
The unbalancedness magnitude at time $t\geq 0$ of the weight vectors and matrices of the CP factorization in Eq.~\eqref{eq:deepcp} is defined as  :
	\begin{equation*}
		\varepsilon(t)=\!\!\!\displaystyle\max\limits_{\substack{r\in\{1,\hdots,R\}, (n,m)\in\{1,\ldots,N\}^2\\ i\in \{1,\ldots,k_m\}}}\!\Big |\|\bm{w}_r^n(t)\|^2-\|\bm{A}_i^{m,r}\!(t)\|^2\Big|.
	\end{equation*}
\end{definition}
Note that this notion of \textit{unbalancedness magnitude}  of the deep CP factorization is inspired from~\citet{razin2021implicit}, where the unbalancedness magnitude~\citep{du2018algorithmic} of the weight vectors of the CP decomposition was introduced. 

Note that $\varepsilon(0)$ is per definition purely determined by the initialization.
We will show later that $\varepsilon(t)$ is constant during the gradient descent optimization, which is crucial to show our first main result.
\begin{theorem}
	\label{th:dy_deepcp}
Assume that $\varepsilon(0)=0$. Then,  for any  $r\in\{1,\hdots,R\}$ and time $t\geq 0$ at which $\left\| \underset{n=1}{\overset{N}\bigotimes} \underset{i=1}{\overset{k_n}\prod}
{\bm{A}}_i^{n,r}(t){\bm{w}}_r^{n}(t)\right\|\neq 0$:
\begin{enumerate}[label=(\roman*),leftmargin=0.6cm]
\itemsep=0.2cm
\item  The weight vectors of the CP factorization in Eq.~\eqref{eq:deepcp} evolve according to
\begin{equation*}
	\!\!\frac{d}{dt}\left\| \bigotimes_{n=1}^N \bm{w}_r^n(t)\right\|\!=\! N \delta_r(t)\left\| \bigotimes_{n=1}^N \bm{w}_r^n(t)\right\|^{2-\frac{2}{N}+\frac{k_1+\ldots +k_N}{N}\!\!\!}, 
\end{equation*}
where
$\delta_r(t) :=  \left\langle -\nabla \mathcal{L}(\mathcal{W}(t)), \underset{n=1}{\overset{N}\bigotimes} 
\underset{i=1}{\overset{k_n}\prod}
\widehat{\bm{A}}_i^{n,r}(t)\widehat{\bm{w}}_r^{n}(t)\right\rangle$, $\widehat{\bm{w}}_r^{n}(t):=\frac{\bm{w}_r^{n}}{\|\bm{w}_r^{n}(t)\|}$ and $\widehat{\bm{A}}_i^{n,r}(t):=\frac{\bm{A}_i^{n,r}(t)}{\|\bm{A}_i^{n,r}(t)\|}$.
\item If in addition $\{\bm{A}_i^{n,r}(0)\}_{r=1\ n=1\ i=1}^{R\ \ \ \ N\ \ \ \ \ k_n}$ are rank-one matrices satisfying $$\bm{A}_i^{n,r}(0)^\top \bm{A}_i^{n,r}(0) = \bm{A}_{i+1}^{n,r}(0) \bm{A}_{i+1}^{n,r}(0)^\top,$$ for all $i\in\{1,\ldots,k_n-1\}$ with $k_n\geq 2$ and $n \in \{1,\ldots,N\}$, then
\begin{equation*}
    \delta_r(t) \!=  \!\left\langle \!\!-\nabla \mathcal{L}(\mathcal{W}(t)), \underset{n=1}{\overset{N}\bigotimes} \frac{\underset{i=1}{\overset{k_n}\prod}{\bm{A}}_i^{n,r}(t){\bm{w}}_r^{n}(t)} {\left\|\underset{i=1}{\overset{k_n}\prod}{\bm{A}}_i^{n,r}(t){\bm{w}}_r^{n}(t)\right\|}\right\rangle  \!\zeta_r(t),
\end{equation*}
where 
$\zeta_r(t) :=\left|  \left\langle \underset{n=1}{\overset{N}\bigotimes}{\tilde{\bm{v}}}_r^{n}(t) , \underset{n=1}{\overset{N}\bigotimes} \widehat{\bm{w}}_r^{n}(t) \right\rangle \right|$ and ${\tilde{\bm{v}}}_r^{n}(t)$
is the first right singular vector of %
$\bm{A}_{k_n}^{n,r}(t)$.

\end{enumerate}

\end{theorem}
By Theorem~\ref{th:dy_deepcp}, if  all the depths $k_1, \ldots, k_N$ are equal to the same value, say $k$, we obtain:
\begin{equation}
	\label{eq:dy_deepcp}
\frac{d}{dt}\left\| \bigotimes_{n=1}^N \bm{w}_r^n(t)\right\|=N \delta_r(t)\left\| \bigotimes_{n=1}^N \bm{w}_r^n(t)\right\|^{2-\frac{2}{N}+k}.
\end{equation}
Note that Part (ii) of Theorem~\ref{th:dy_deepcp} allows us to express $\delta_r(t)$ defined in Part (i)  in terms of quantities with norms that do not depend with depths. This is achieved by characterizing the evolution of the singular values of the product of the matrix parameters $\bm{A}_i^{n,r}(t)$ (see Lemma~\ref{lem:sv}).

This shows that the evolution rates of the norm of the blocks of the tensor $\sum_{r=1}^R  \underset{n=1}{\overset{N}\bigotimes} \bm{w}_r^n$ are proportional to their norm raised to the power of $2-\frac{2}{N}+k$. 
This is in line with the observation that CP block norms move faster when large and slower when small, as reported by~\citet{razin2021implicit}. More interestingly, this effect is more pronounced with larger depths.
When the depth $k$ increases, one block $r$, likely the one whose norm is maximum, will see its norm increases significantly faster~(the bigger the value of $k$ the faster) than the norm of all other blocks, up to a stability stage when $\delta_r(t)$ converges towards 0, meaning this block is somehow optimized.     

This sequential block optimization mechanism promotes low-rankness in a greedy fashion where the blocks are selected and optimized one after the other, which  will be confirmed experimentally. Interestingly, similar observations were made in~\citet{arora2019implicit,li2021towards,ge2021understanding}.
An interesting property of deep CP factorization  that our theoretical analysis reveals, lies in that the effect of the depth of the factorization on the implicit regularization is \textit{polynomial}, while it is \textit{quadratic} in deep matrix and \enquote*{shallow} CP decomposition, as shown in Theorems~\ref{th:dy_deepm} and~\ref{th:dy_cp}. 

One consequence of the greedy sequential optimization of the blocks is that when a sufficient number of blocks of $\sum_{r=1}^R  \underset{n=1}{\overset{N}\bigotimes} \bm{w}_r^n$ are effectively used (having a norm significantly different from zero) and have been optimized the other block remain ignored with a very small norm. Also, the number of effective blocks quickly decreases with increasing depth $k$.

Yet, Theorem~\ref{th:dy_deepcp} does not explicitly specify the effect of the implicit regularization on the learned tensor $\mathcal{W}$ by the deep CP factorization.
The following theorem shows that the dynamical characterization provided above would favor selecting only a few number of the blocks of $\mathcal{W}$, promoting low canonical rank solutions.
\begin{theorem}
\label{th:rank_deepcp}
Assume that $\varepsilon(0)=0$. Then,  for any  time $t\geq 0$ of the optimization  of the CP factorization in Eq.~\eqref{eq:deepcp},  the following inequality holds for all $r\in\{1,\ldots,R\}$:
\begin{equation*}
\left\| \underset{n=1}{\overset{N}\bigotimes} 
\underset{i=1}{\overset{k_n}\prod}
{\bm{A}}_i^{n,r}(t){\bm{w}}_r^{n}(t)\right\| \leq \left \|\bigotimes_{n=1}^N  \bm{w}_r^n(t)\right\|^{1+\frac{k_1+\ldots +k_N}{N}}.
\end{equation*}
\end{theorem}
If  $k_1 = \ldots = k_N$, let us say $k$,  then:
\begin{equation}
	\label{eq:rank_deepcp}
\left\| \underset{n=1}{\overset{N}\bigotimes} 
\underset{i=1}{\overset{k}\prod}
{\bm{A}}_i^{n,r}(t){\bm{w}}_r^{n}(t)\right\| \leq \left \|\bigotimes_{n=1}^N  \bm{w}_r^n(t)\right\|^{1+k}.
\end{equation}
Let us denote by $N_1(t)$ the number of blocks with non zero norm of the factorized tensor $\mathcal{W} =\sum_{r=1}^R  \bigotimes_{n=1}^N  \prod_{i=1}^{k_n}  \bm{A}_i^{n,r} \bm{w}_r^n$ at iteration $t$, and let $N_2(t)$ be the number of blocks with non zero norm of the factorized tensor $\sum_{r=1}^R  \bigotimes_{n=1}^N  \bm{w}_r^n$  at iteration $t$.
 Theorem~\ref{th:rank_deepcp} shows that the term $\left \|\bigotimes_{n=1}^N  \bm{w}_r^n(t)\right\|$ controls the norm of the $r$-th block of the deep CP factorization. If it is close to zero, then $\left\| \underset{n=1}{\overset{N}\bigotimes} 
\underset{i=1}{\overset{k_n}\prod}
{\bm{A}}_i^{n,r}(t){\bm{w}}_r^{n}(t)\right\|$ is close to zero as well. 
Then, whatever the iteration $t$, $N_1(t) \leq N_2(t)$. Considering that $\mathrm{rank}(\mathcal{W}) \leq N_1(t)$, we can conclude that the rank of the learned  tensor $\mathcal{W}$ is bounded by $N_2(t)$.
Putting all together Theorem~\ref{th:dy_deepcp}  shows that the depth $k$ of the factorization ensures the convergence of the tensor $\sum_{r=1}^R  \bigotimes_{n=1}^N  \bm{w}_r^n$ towards a tensor with a small $N_2(t)$ value,  hence leads the deep CP factorization to converge to a tensor $\mathcal{W}$ that has a low canonical rank.

\section{Proof Overview}
\label{sec:proofs}

To prove Theorem~\ref{th:dy_deepcp}, we need the following result.
\begin{lemma}
	\label{lem:dy}
For all $r\in\{1,\ldots,R\}$, $n,m\in\{1,\ldots, N\}$, $i\in\{1,\ldots,k_n\}$ and $j\in\{1,\ldots,k_m\}$, the followings hold $\forall t \geq 0$,
\begin{enumerate}[(i)]
	\item $\|\bm{A}_i^{n,r}\!(t)\|^2\!-\!\|\bm{A}_j^{m,r}\!(t)\|^2=\|\bm{A}_i^{n,r}\!(0)\|^2-\|\bm{A}_j^{m,r}\!(0)\|^2$,
	\item $\|\bm{w}_r^n(t)\|^2-\|\bm{w}_r^m(t)\|^2=\|\bm{w}_r^n(0)\|^2-\|\bm{w}_r^m(0)\|^2$,
	\item $ \|\bm{w}_r^n(t)\|^2-\|\bm{A}_i^{n,r}(t)\|^2=\|\bm{w}_r^n(0)\|^2-\|\bm{A}_i^{n,r}(0)\|^2$. 
\end{enumerate}
\end{lemma}
The proof of Lemma~\ref{lem:dy} is in Appendix~\ref{app:proofs}. The essential idea of the proof comes from~\citet{razin2021implicit}.
The key steps are as follows. We first show that $\displaystyle\frac{d}{dt}\|\bm{A}_i^{n,r}(t)\|^2$ is independent of $n$ and $i$. So $\forall (n,m)\in \llbracket1,N\rrbracket^2$ and $\forall (i,j)\in \llbracket1,k_n\rrbracket\times \llbracket1,k_m\rrbracket$, the derivatives of  $\|\bm{A}_i^{n,r}(t)\|^2$ and $\|\bm{A}_j^{m,r}(t)\|^2$ are equal, which means that $\|\bm{A}_i^{n,r}(t)\|^2-\|\bm{A}_j^{m,r}(t)\|^2$ does not vary with $t$. 
The assertion $(ii)$  is shown by the same arguments as above. 
The proof of~$(iii)$ is based on the observation that, $\forall n,m\in \llbracket1,N\rrbracket$ and $\forall i\in \llbracket1,k_n\rrbracket$,
$\displaystyle\frac{d}{dt}\|\bm{w}_r^n(t)\|^2=\frac{d}{dt}\|\bm{A}_i^{n,r}(t)\|^2$.

We now present a sketch of the proof of Theorem~\ref{th:dy_deepcp}. Full
details of the proof is provided in Appendix~\ref{app:proofs}. The ideas of the proof are similar to the ideas in~\citet{razin2021implicit}, with the difference that the extension to deep CP factorization will result in some technical complications due to the presence of the weight matrices $\bm{A}_i^{n,r}(t)$. 
\paragraph{Sketch of the proof of Theorem~\ref{th:dy_deepcp}.}
The main step of the proof of part (i) of the theorem is to show that 
\begin{equation}
	\label{eq:sktech_dy}
	\frac{d}{dt}\!\!\left\|\underset{n=1}{\overset{N}\bigotimes}\bm{w}_r^n(t)\right\|\!=\! N \delta_r(t)\!\left\|\underset{n=1}{\overset{N}\bigotimes}\bm{w}_r^n(t)\right\|^{2-\frac{2}{N}}\!\!\!\!\prod_{n=1}^N \!\prod_{i=1}^{k_{n}}\!\|\bm{A}_i^{n,r}\!(t)\|.
\end{equation}
This result is obtained by deriving  upper and lower bounds of the first term in~\eqref{eq:sktech_dy}, which converges to the same value when the unbalancedness magnitude assumption is satisfied. 
Then, using Lemma~\ref{lem:dy} and the assumption $\epsilon(0)=0$, we prove that
$
\|\bm{A}_i^{n,r}(t)\|=\left\|\underset{m=1}{\overset{N}\bigotimes}\bm{w}_r^m(t)\right\|^{\frac{1}{N}}
$.
Plugging the result into Eq.~\eqref{eq:sktech_dy} completes the proof.

To prove part (ii) of the theorem, we state  Lemma~\ref{lem:sv} which characterizes the evolution of the singular values of the product matrix $\displaystyle\prod_{i=1}^{k_{n}}\bm{A}_i^{n,r}(t)$. The proof of this lemma is inspired by Theorem 3 of~\citet{arora2019implicit}.
This allows us to express $\left\|\underset{i=1}{\overset{k_n}\prod}{\bm{A}}_i^{n,r}(t){\bm{w}}_r^{n}(t)\right\|$ in terms of $\left(\underset{i=1}{\overset{k_n}\prod}\|{\bm{A}}_i^{n,r}(t)\|\right) \|{\bm{w}}_r^{n}(t)\|$.

\paragraph{Proof of Theorem~\ref{th:rank_deepcp}.}
Let us first recall that $
\|\bm{A}_i^{n,r}(t)\|=\left\|\underset{m=1}{\overset{N}\bigotimes}\bm{w}_r^m(t)\right\|^{\frac{1}{N}}$, when $\epsilon(0)=0$. We have,
\begin{align*}
	&\left\|\underset{n=1}{\overset{N}\bigotimes}\prod_{i=1}^{k_{n}}\bm{A}_i^{n,r}(t)\bm{w}_r^{n}(t)\right\|
	=\prod_{n=1}^N\left\|\prod_{i=1}^{k_{n}}\bm{A}_i^{n,r}(t)\bm{w}_r^{n}(t)\right\|\\
	&\qquad\qquad \leq \prod_{n=1}^N\left\|\prod_{i=1}^{k_{n}}\bm{A}_i^{n,r}(t)\right\| \left\|\bm{w}_r^{n}(t)\right\|\\
	&\qquad\qquad \leq\prod_{n=1}^N\left(\prod_{i=1}^{k_{n}}\|\bm{A}_i^{n,r}(t)\|\right)\left\|\bm{w}_r^{n}(t)\right\| \\
	&\qquad\qquad =\prod_{n=1}^N\left(\prod_{i=1}^{k_{n}}\left\|\underset{m=1}{\overset{N}\bigotimes}\bm{w}_r^{m}(t)\right\|^{\frac{1}{N}}\right)\prod_{n=1}^N\left\|\bm{w}_r^{n}(t)\right\| \\
	&\qquad\qquad = \left\|\underset{n=1}{\overset{N}\bigotimes}\bm{w}_r^{n}(t)\right\|^{\frac{k_1+\ldots+k_N}{N}}  \left\|\underset{n=1}{\overset{N}\bigotimes}\bm{w}_r^{n}(t)\right\| \\
	&\qquad\qquad = \left\|\underset{n=1}{\overset{N}\bigotimes}\bm{w}_r^{n}(t)\right\|^{1+\frac{k_1+\ldots+k_N}{N}}.
\end{align*}

\section{Experimental Analysis}
\label{sec:xp}
We present here an experimental analysis  that helps understanding our main theoretical results. We first detail the experimental settings and investigate the main trends that we observed during learning. %
We then report a more extensive analysis of the low rank inducing feature of deep over-parameterized tensor optimization.  
Finally, we explore how and when depth may help improving loss minimization.

\subsection{Experimental Settings}
We focus on tensor completion task: given a partially observed tensor $\mathcal{A}$, we learn a model to match inputs indices~(tuples of size $N$, the order of $\mathcal{A}$) to the observed values. We generated synthetic data in order to analyze and control 
the phenomena of implicit regularization in tensor factorization.
	
\textbf{Synthetic data.} We generated a $10\times 10\times 10\times 10$ tensor with CP-rank equal to $5$ and entries sampled from a centered and reduced normal distribution. From this complete tensor, we randomly split the set of (indices, values) to build training and testing sets. In the following, we comment experiments for various ratio of observed values (from $25\%$ to $10\%$).

\textbf{Model initialization.}
As we just said, entries of ${\bm{w}}_r^{n}$ are sampled from a centered normal distribution with a %
small variance $\sigma_w$. In our deep CP formulation, we also need to initialize entries of $\bm{A}_i^{n,r}$ from a centered normal distribution with small variance $\sigma_A$. Alike in previous works studying implicit regularization, we remarked the sensitivity of implicit regularization to model's initialization. With a too large $\sigma_w$ the model does not converge to a low-rank solution, while with a too small $\sigma_w$ a solution might exists, however, after a prohibitive number of gradient descent iterations. With our deep formulation, the product of multiple - small norm - matrices may lead to numerical instabilities and/or to the well known vanishing gradient. However, we observed implicit regularization with highest values in matrix initialization.
In \citet{jing2020implicit}, authors empirically found that standard Kaiming (He) initialization \citep{he2015delving} of multiple matrices stacked after the encoder is able to yield implicit regularization in the latent space. In order to be close to theoretical conditions, we use in our simulations zero mean and near zero standard deviation for initializing the parameters to be learned. 
We have also considered matrices that are initialized with small values except on the diagonal to speed up learning.

\subsection{Investigating Learning Behaviour}
First, we investigate typical phenomenon that we observed during the learning. In all the experiments that we report here the percentage of observed and unobserved inputs in the tensor are 20\% and 80\%. 
Importantly, the learning is very sensitive to initial settings, meaning whatever the depth the learning may converge towards low-rank or high-rank solutions (see Figure \ref{fig:multiruns}).

\begin{figure}[!h]
 	\centering
	\hspace*{-0.3cm}\includegraphics[trim=0pt 0pt 0pt 22pt,clip, scale=0.42]{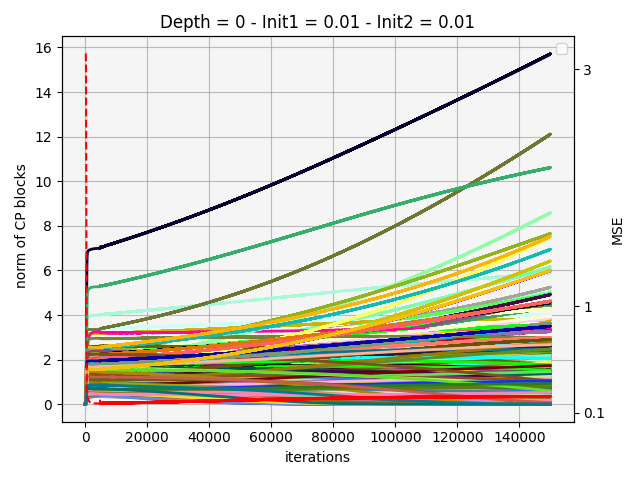}\\[-0.2cm]
 (a)\\[0.3cm]
	\includegraphics[trim=0pt 0pt 0pt 24pt,clip, scale=0.44]{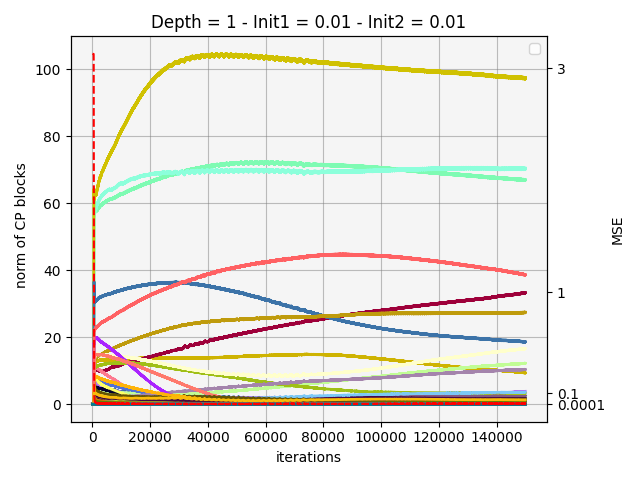}\\[-0.2cm] (b)\\[0.3cm]
	\includegraphics[trim=0pt 0pt 0pt 26pt,clip, scale=0.44]{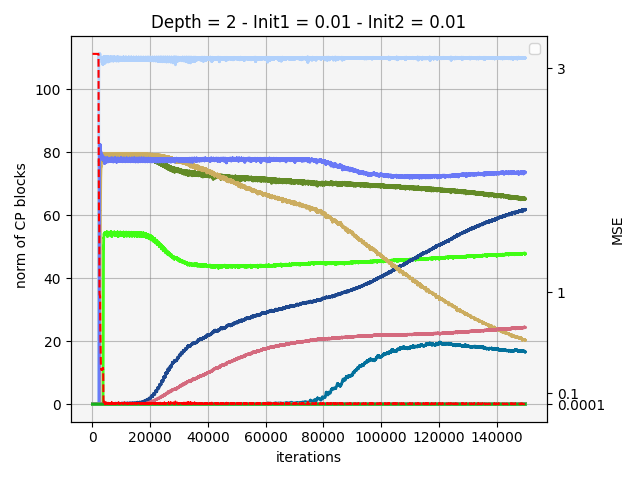}\\[-0.2cm] (c)
	\caption{
This series of figures compares the learning behaviour for shallow to deep tensors through the evolution of the norms of the blocks along training epochs for depth 0 (a), 1 (b), and 2 (c) for a particular initialization ($\sigma_w = 0.01$, $\sigma_A = 0.01$). Each figure shows 500 curves corresponding to the norm of blocks (y-scale on the left of the plot), plus an additional curve (dotted red line)  which stands for the loss (with y-scale on the right of the plot in log-scale). }
 		\label{fig:Norm_vs_epcoh_perdepth1}
\end{figure}

Second, Figures \ref{fig:Norm_vs_epcoh_perdepth1} and \ref{fig:Norm_vs_epcoh_perdepth2} show typical learning behaviour with respect to depth, for two different initialization settings, where one sees in both cases that shallow architectures tends to converge to high-rank solutions with many blocks exhibiting a non negligible norm, while increasing depth makes most blocks' norm converge to 0 yielding a much more relevant low-rank solution.

\begin{figure}[!]
 \centering
   \includegraphics[trim=0pt 0pt 0pt 12pt,clip, scale=0.84]{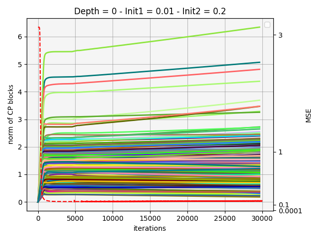}\\[-0.2cm]
   (a)\\[0.3cm]

   \includegraphics[trim=0pt 0pt 0pt 12pt,clip, scale=0.84]{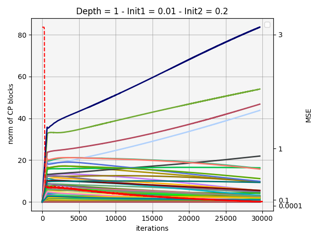}\\[-0.2cm]
   (b)\\[0.3cm]	

	\includegraphics[trim=0pt 0pt 0pt 12pt,clip, scale=0.84]{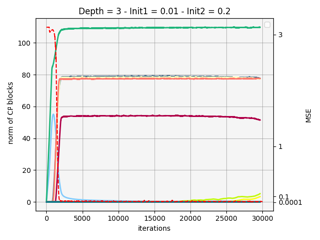}\\[-0.2cm]
	(c)\\[0.3cm]

	\hspace*{-0.3cm}\includegraphics[trim=0pt 0pt 0pt 12pt,clip, scale=0.79]{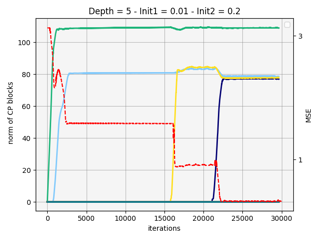}\\[-0.2cm]
	(d)
	\caption{
This series of figures compares the learning behaviour for shallow to deep tensors through the evolution of the norms of the blocks along training epochs for depth 0 (a), 1 (b), 3 (c), and 5~(d) for a particular initialization ($\sigma_w = 0.005$ and $\sigma_A = 0.1$). 
}
 		\label{fig:Norm_vs_epcoh_perdepth2}
\end{figure}

Third, one may note (see e.g. Figures~\ref{fig:Norm_vs_epcoh_perdepth1}~(c) and~\ref{fig:Norm_vs_epcoh_perdepth2}~(d)) that blocks emerge sequentially, in a greedy fashion along the training process, one at a time. This phenomenon has already been observed in e.g.~\citet{razin2020implicit} and is a consequence of the dynamics rule in Theorem~\ref{th:dy_deepcp}, as discussed in Section~\ref{sec:main_results}. We also observed  blocks whose norm rises suddenly then decrease to converge to their final norm which may be also a consequence of the polynomial dynamics as well, this is illustrated in Figure \ref{fig:Norm_vs_epcoh_perdepth1}.

\subsection{How Depth Yields Low-Rank Solutions}
  
We summarize in Figure~\ref{fig:multiruns} a number of experiments that illustrate the effectiveness of deeper architectures to consistently converge to low-rank solutions whose rank is close to the true tensor rank, whatever the initial conditions. We launched more than 1500 learning experiments using various initialization parameters and seeds, for depth ranging from 0 to 5. We report for each experiment the effective CP-rank of the model, i.e. the number of blocks that emerged at convergence (blocks that have a norm greater than $1$). Again in all experiments reported here the percentage of observed and unobserved inputs in the tensor are 20\% and 80\%. 

 \begin{figure}[t]
 	\centering
	\hspace*{-0.3cm}\includegraphics[scale=0.54]{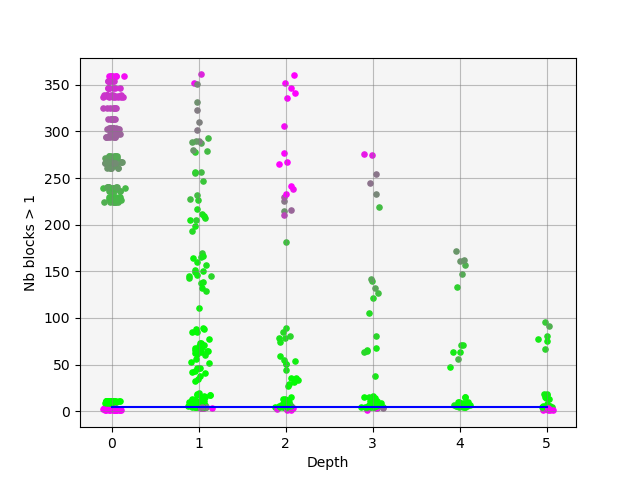}
	\caption{Analysis of the impact of the depth on the rank of the learned tensor. The figure shows for shallow (on the left, depth=0) to deep models (up to depth=5, on the right) the effective rank of the learned tensors for a number of runs that differ from initialization setting. Every single point stands for a learning experiment. Points are plotted with a small random displacement in x and y coordinates to better see point clouds. The color represents the test loss of the model, from green to purple respectively for small to large loss (ranging from $5.10^{-5}$ to $3.35$). Few runs diverged and some led to 0 rank due to lacks of iterations (11 over 1500), those runs are not reported here. The blue line shows the real tensor rank (5) to approximate. Runs below this line lead to high losses. Conversely, most of runs which  converged to higher ranks are able to minimize the loss objective.}
 		\label{fig:multiruns}
\end{figure}

One can observe much more variability on the learned tensor rank when the depth is limited. For shallower architectures, the impact of initialization is huge and the solutions are mostly of high rank. For deeper architectures, the learned tensor rank is much more stable, close to the true tensor rank, hence showing lower dependency to the initialization setting.

 \begin{figure}[t]
 	\centering
\hspace*{-0.3cm}\includegraphics[trim={0 0 0 0},clip,scale=0.535]{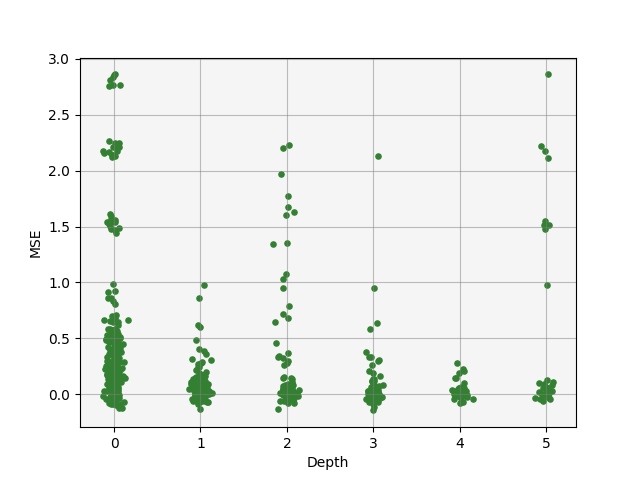}
\caption{Analysis of the impact of the depth (x-axis) on the generalization loss (y-axis). The figure shows from shallow (on the left, depth=0) to deep models (up to depth=5, on the right), the generalization losses achieved in a number of runs that differ from initialization setting. Every single point stands for a learning experiment. Points are plotted with a small random displacement in x and y coordinates to better see point clouds.}
\label{fig:multiruns_losses}
\end{figure}

\subsection{Impact of the Depth on Generalization Loss}

Finally, we explore how and when the depth may help achieving generalization. Figure~\ref{fig:multiruns_losses} is a similar plot as  Figure~\ref{fig:multiruns} but where the y-axis stands for the generalization loss. We again run many experiments for depth from 0 to 5 and for various initialization settings. In this figure all experiments reported have been obtained using percentages of observed and unobserved inputs equal to 20\% and 80\%. 

As may be observed, whatever the depth, a small generalization loss may be achieved. However, increasing the depth makes the optimization much more robust and stable with respect to initialization. Depth consistently helps reaching best achievable generalization loss whatever the initialization. Hence, increasing the depth allows reaching low-rank approximation as well as low generalization loss.

To go deeper in the analysis, Table \ref{tab:loss} reports for the depth ranging from 0 to 5, and for a percentage of unobserved values ranging from 75\% to 90\%, the smallest loss obtained on validation data whatever the initialization, and the rank of the corresponding learned tensor (in brackets). As we have already discussed,  when running many experiments with various initialization settings one may most often get a low-rank solution  whatever the depth (see bottom points for every depth in Figure~\ref{fig:multiruns}). This explains why many of these best performing solutions are of low rank, whatever the depth (e.g. first line with depth=0) and the percentage of unobserved inputs. Yet, these results show that depth often allows reaching low-rank solutions as well as low  generalization loss, and thus achieving a good trade-off  between tensor rank and generalization error.

\begin{table}[t]
\begin{center}
\begin{tabular}{cccc} 
\hline
\small
& 75 & 85 & 90 \\
\hline
0 &
 7.925e-05 (5) &
 2.598e-05 (14) &
 5.836e-06 (11) \\
1 &
 1.479e-05 (67) &
 2.685e-04 (12) &
  1.42e-05 (7) \\
2 &
 3.607e-06 (10) &
 5.215e-05 (17) &
  2.37e-04 (13) \\
3 &
 2.073e-06 (13) &
 8.048e-04 (10) &
 2.933e-04 (6) \\
4 &
 8.053e-04 (5) &
 2.238e-04 (6) &
 2.350e-04 (9) \\
5 &
 8.116e-03 (5) &
 8.169e-01 (3) &
 5.265e-01 (4) \\
\hline
\end{tabular}
\caption{Comparison of best performing tensor factorization for a number of cases corresponding to few architecture depths and to  various percentages of missing tensor inputs at training time, ranging from 75\% to 90\%. The best generalization and the effective rank of the best tensor factorization (in brackets) are reported.}
\label{tab:loss}
\end{center}
\end{table}

\subsection{Real-World Data}
\label{subsec:realdata}
 We also run experiments on two real data sets. We used Meteo-UK\footnote{\url{https://www.metoffice.gov.uk/research/climate/maps-and-data/historic-station-data}.} 
and CCDS\footnote{\url{https://viterbi-web.usc.edu/~liu32/data.html}.} data sets~\citep{lozano2009spatial}, which contain monthly measurements of temporal variables in various stations across UK and North America, resulting in tensors of dimension $(50, 16, 5)$ and $(50, 15, 25)$, respectively.
Figure~\ref{fig:realdata}  shows completion performance with $30\%$ of observed data for multiple runs varying with initialization std in $[0.0005, \ldots, 0.001]$. 
The obtained experimental results on real-world data corroborate the simulations and confirm our theoretical findings.
We also used in this experiment random initialization with rank-one matrices and observed that the same experimental behaviour is reproduced~(see Appendix~\ref{app:xp}).

\begin{figure*}[t]
 	\centering
 	 \begin{minipage}[b]{0.49\linewidth}
	\includegraphics[scale=0.53]{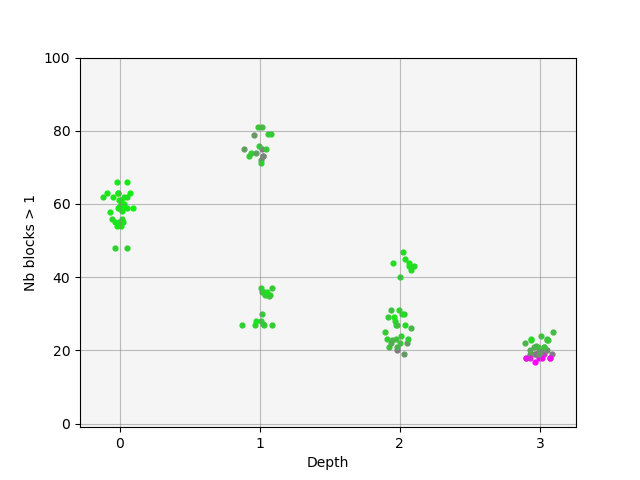}
	\end{minipage}
	\begin{minipage}[b]{0.49\linewidth}
	\includegraphics[scale=0.53]{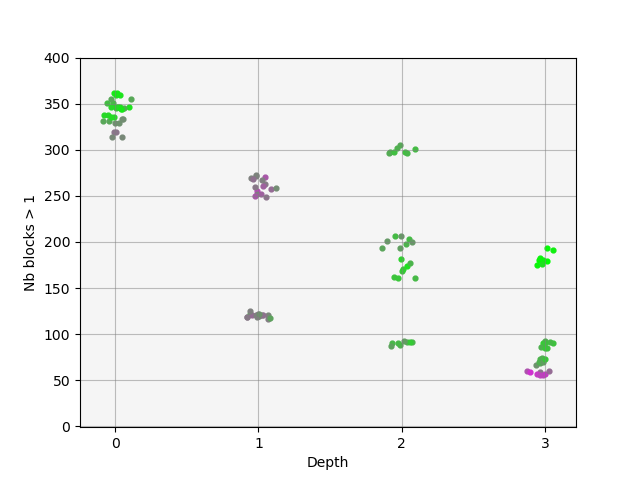}
	\end{minipage}
	\caption{Tensor completion using Meteo-UK (left) and CCDS (right) data sets: analysis of the impact of the depth on the rank of the learned tensor. The figure shows for shallow (on the left, depth=0) to deep models (up to depth=5, on the right) the effective rank of the learned tensors for a number of runs that differ from initialization setting. Every single point stands for a learning experiment. Points are plotted with a small random displacement in x and y coordinates to better see point clouds. The color represents the test loss of the model, from green to purple respectively for small to large loss (ranging from $0.3$ to $1.1$ for Meteo-UK and from $0.59$ to $1.12$ for CCDS).} 
 		\label{fig:realdata}
\end{figure*}

\section{Conclusion}

We provided a theoretical analysis of
implicit regularization in deep tensor factorization building on previous advances studying tensor 
and deep matrix factorization. Our results suggest a form of greedy low tensor rank search, but where the impact of implicit regularization is polynomially dependent of the depth. Experiments confirmed our theoretical results and provided insights on the main role of initialization, especially for shallow architectures. While shallow architectures may converge to low-rank and high-rank solutions, deeper factorizations consistently converge to low-rank solutions, close to the true tensor rank. Finally, deeper architectures seem to help reaching more consistently best performing solutions with respect to the generalization error.

\section*{Acknowledgements}
We thank the reviewers and the meta-reviewer for their helpful comments and for suggesting a way to improve Theorem~\ref{th:dy_deepcp}.
Research reported in this paper was partially supported by PHC Utique no.~44318NJ granted by the Ministry of Higher Education and Scientific Research of Tunisia and the Ministry of Foreign Affairs in France.


\bibliography{implicit_regularization}

\begin{thebibliography}{22}
\providecommand{\natexlab}[1]{#1}
\providecommand{\url}[1]{\texttt{#1}}
\expandafter\ifx\csname urlstyle\endcsname\relax
  \providecommand{\doi}[1]{doi: #1}\else
  \providecommand{\doi}{doi: \begingroup \urlstyle{rm}\Url}\fi

\bibitem[Arora et~al.(2018)Arora, Cohen, and Hazan]{arora2018optimization}
Arora, S., Cohen, N., and Hazan, E.
\newblock On the optimization of deep networks: Implicit acceleration by
  overparameterization.
\newblock In \emph{International Conference on Machine Learning~(ICML)}, pp.\
  244--253, 2018.

\bibitem[Arora et~al.(2019)Arora, Cohen, Hu, and Luo]{arora2019implicit}
Arora, S., Cohen, N., Hu, W., and Luo, Y.
\newblock Implicit regularization in deep matrix factorization.
\newblock \emph{Advances in Neural Information Processing Systems~(NeurIPS)},
  2019.

\bibitem[De~Lathauwer et~al.(2000)De~Lathauwer, De~Moor, and
  Vandewalle]{de2000multilinear}
De~Lathauwer, L., De~Moor, B., and Vandewalle, J.
\newblock A multilinear singular value decomposition.
\newblock \emph{SIAM journal on Matrix Analysis and Applications}, 21\penalty0
  (4):\penalty0 1253--1278, 2000.

\bibitem[Du et~al.(2018)Du, Hu, and Lee]{du2018algorithmic}
Du, S., Hu, W., and Lee, J.~D.
\newblock Algorithmic regularization in learning deep homogeneous models:
  Layers are automatically balanced.
\newblock \emph{Advances in Neural Information Processing Systems~(NeurIPS)},
  31, 2018.

\bibitem[Gandy et~al.(2011)Gandy, Recht, and Yamada]{gandy2011tensor}
Gandy, S., Recht, B., and Yamada, I.
\newblock Tensor completion and low-n-rank tensor recovery via convex
  optimization.
\newblock \emph{Inverse problems}, 27\penalty0 (2):\penalty0 025010, 2011.

\bibitem[Ge et~al.(2021)Ge, Ren, Wang, and Zhou]{ge2021understanding}
Ge, R., Ren, Y., Wang, X., and Zhou, M.
\newblock Understanding deflation process in over-parametrized tensor
  decomposition.
\newblock \emph{Advances in Neural Information Processing Systems~(NeurIPS)},
  34, 2021.

\bibitem[Grasedyck et~al.(2013)Grasedyck, Kressner, and
  Tobler]{grasedyck2013literature}
Grasedyck, L., Kressner, D., and Tobler, C.
\newblock A literature survey of low-rank tensor approximation techniques.
\newblock \emph{GAMM-Mitteilungen}, 36\penalty0 (1):\penalty0 53--78, 2013.

\bibitem[Gunasekar et~al.(2017)Gunasekar, Woodworth, Bhojanapalli, Neyshabur,
  and Srebro]{gunasekarimplicit}
Gunasekar, S., Woodworth, B., Bhojanapalli, S., Neyshabur, B., and Srebro, N.
\newblock Implicit regularization in matrix factorization.
\newblock \emph{Advances in Neural Information Processing Systems~(NeurIPS)},
  2017.

\bibitem[He et~al.(2015)He, Zhang, Ren, and Sun]{he2015delving}
He, K., Zhang, X., Ren, S., and Sun, J.
\newblock Delving deep into rectifiers: Surpassing human-level performance on
  imagenet classification.
\newblock In \emph{Proceedings of the IEEE international conference on computer
  vision}, pp.\  1026--1034, 2015.

\bibitem[Jing et~al.(2020)Jing, Zbontar, et~al.]{jing2020implicit}
Jing, L., Zbontar, J., et~al.
\newblock Implicit rank-minimizing autoencoder.
\newblock \emph{Advances in Neural Information Processing Systems}, 33, 2020.

\bibitem[Kolda \& Bader(2009)Kolda and Bader]{Kolda}
Kolda, T.~G. and Bader, B.~W.
\newblock Tensor decompositions and applications.
\newblock \emph{SIAM Review}, 51\penalty0 (3):\penalty0 455--500, 2009.

\bibitem[Kumar \& Poole(2020)Kumar and Poole]{kumar2020implicit}
Kumar, A. and Poole, B.
\newblock On implicit regularization in $\beta$-{VAE}s.
\newblock In \emph{International Conference on Machine Learning~(ICML)}, 2020.

\bibitem[Li et~al.(2021)Li, Luo, and Lyu]{li2021towards}
Li, Z., Luo, Y., and Lyu, K.
\newblock Towards resolving the implicit bias of gradient descent for matrix
  factorization: Greedy low-rank learning.
\newblock In \emph{International Conference on Learning Representation~(ICLR)},
  2021.

\bibitem[Lozano et~al.(2009)Lozano, Li, Niculescu-Mizil, Liu, Perlich, Hosking,
  and Abe]{lozano2009spatial}
Lozano, A.~C., Li, H., Niculescu-Mizil, A., Liu, Y., Perlich, C., Hosking, J.,
  and Abe, N.
\newblock Spatial-temporal causal modeling for climate change attribution.
\newblock In \emph{Proceedings of the 15th ACM SIGKDD international conference
  on Knowledge discovery and data mining}, pp.\  587--596, 2009.

\bibitem[Milanesi et~al.(2021)Milanesi, Kadri, Ayache, and
  Arti{\`e}res]{milanesi2021implicit}
Milanesi, P., Kadri, H., Ayache, S., and Arti{\`e}res, T.
\newblock Implicit regularization in deep tensor factorization.
\newblock In \emph{International Joint Conference on Neural Networks~(IJCNN)},
  2021.

\bibitem[Neyshabur et~al.(2014)Neyshabur, Tomioka, and
  Srebro]{neyshabur2014search}
Neyshabur, B., Tomioka, R., and Srebro, N.
\newblock In search of the real inductive bias: On the role of implicit
  regularization in deep learning.
\newblock \emph{arXiv preprint arXiv:1412.6614}, 2014.

\bibitem[Oseledets(2011)]{oseledets2011tensor}
Oseledets, I.~V.
\newblock Tensor-train decomposition.
\newblock \emph{SIAM Journal on Scientific Computing}, 33\penalty0
  (5):\penalty0 2295--2317, 2011.

\bibitem[Razin \& Cohen(2020)Razin and Cohen]{razin2020implicit}
Razin, N. and Cohen, N.
\newblock Implicit regularization in deep learning may not be explainable by
  norms.
\newblock \emph{Advances in Neural Information Processing Systems~(NeurIPS)},
  2020.

\bibitem[Razin et~al.(2021)Razin, Maman, and Cohen]{razin2021implicit}
Razin, N., Maman, A., and Cohen, N.
\newblock Implicit regularization in tensor factorization.
\newblock In \emph{International Conference on Machine Learning~(ICML)}, 2021.

\bibitem[Razin et~al.(2022)Razin, Maman, and Cohen]{razin2022implicit}
Razin, N., Maman, A., and Cohen, N.
\newblock Implicit regularization in hierarchical tensor factorization and deep
  convolutional neural networks.
\newblock \emph{arXiv preprint arXiv:2201.11729}, 2022.

\bibitem[Song et~al.(2019)Song, Ge, Caverlee, and Hu]{song2019tensor}
Song, Q., Ge, H., Caverlee, J., and Hu, X.
\newblock Tensor completion algorithms in big data analytics.
\newblock \emph{ACM Transactions on Knowledge Discovery from Data (TKDD)},
  13\penalty0 (1):\penalty0 1--48, 2019.

\bibitem[Zou et~al.(2021)Zou, Wu, Braverman, Gu, Foster, and
  Kakade]{zou2021benefits}
Zou, D., Wu, J., Braverman, V., Gu, Q., Foster, D.~P., and Kakade, S.
\newblock The benefits of implicit regularization from {SGD} in least squares
  problems.
\newblock \emph{Advances in Neural Information Processing Systems~(NeurIPS)},
  2021.

\end{thebibliography}
\bibliographystyle{icml2022}

\newpage
\appendix
\onecolumn
\section{Proofs}
\label{app:proofs}

We provide here the proofs of Lemma~\ref{lem:dy} and Theorem~\ref{th:dy_deepcp}. First let us recall that we consider learning a tensor $\mathcal{W}$ which has the following form $${\mathcal{W}}=\displaystyle\sum_{r=1}^R\underset{n=1}{\overset{N}\bigotimes}\Big (\prod_{i=1}^{k_n}\bm{A}_i^{n,r}\bm{w}_r^n\Big ),\quad \bm{A}_i^{n,r}\in \mathbb{R}^{d_n\times d_n},  \bm{w}_r^{n}\in \mathbb{R}^{d_n},$$
by minimizing the loss function $\mathcal{L}(\mathcal{W}) = \Phi\Big(\{\bm{w}_r^n\}_{r=1\ n=1}^{\ R\ \ \ N},\ \{\bm{A}_i^{n,r}\}_{r=1\ n=1\ i=1}^{\ R\ \ \ N\ \ k_n}\Big )$ using gradient descent.
Then with infinitesimally small learning rate and  non-zero initialization, we have 
$$\frac{d}{dt}\bm{w}_r^n(t)=-\frac{\partial}{\partial \bm{w}_r^n}\Phi \Big(\{\bm{w}_{r'}^{n'}(t)\}_{r'=1\ n'=1}^{\ R\ \ \ \ N},\ \{\bm{A}_i^{n',r'}(t)\}_{r'=1\ n'=1\ i=1}^{\ R\ \ \ \ N\ \ \ k_n'}\Big),$$
and
$$\frac{d}{dt}\bm{A}_i^{r,n}(t)=-\frac{\partial}{\partial\bm{A}_{i}^{r,n}}\Phi \Big(\{\bm{w}_{r'}^{n'}(t)\}_{r'=1\ n'=1}^{\ R\ \ \ \ N},\ \{\bm{A}_i^{n',r'}(t)\}_{r'=1\ n'=1\ i=1}^{\ R\ \ \ \ N\ \ \ k_n'}\Big).$$

To show Lemma~\ref{lem:dy}, we will use the following result shown in~\citet{razin2021implicit}.
	\begin{lemma}\label{lem1}
		$\forall \mathcal{A}\in \mathbb{R}^{d_1\times \ldots \times d_N}$ and $\{\bm{w}^n\in \mathbb{R}^{d_n}\}_{n=1}^N$ where $d_1, \ldots d_N\in \mathbb{N}$, it holds that
		$$\left\langle \mathcal{A}, \underset{n'=1}{\overset{N}\bigotimes}\bm{w}^{n'}\right\rangle=\left\langle [\mathcal{A}]_{(n)}\cdot \underset{n'\ne n}{\odot}\bm{w}^{n'},\bm{w}^n\right\rangle,\ n=1,\hdots,N$$
		where $[\mathcal{A}]_{(n)}$ is matricization of the tensor $\mathcal{A}$ in the mode $n$, and $\odot$ is the kronecker product.
	\end{lemma}

\subsection{Proof of Lemma~\ref{lem:dy}}
\subsubsection{Proof of $(i)$}

We compute $\displaystyle\frac{d}{dt}\|\bm{A}_i^{n,r}(t)\|^2$. We assume that  $\{\bm{w}_r^n\}_{r=1\ n=1}^{\ R\ \ \ N}$ and $\{\bm{A}_j^{n',s}\}_{(j,n',s)\ne (i,n,r)}$ are fixed, and consider
$\Phi_{i,n,r}(\bm{A}_i^{n,r})=\Phi\Big (\{\bm{w}_r^{n'}\}_{r=1\ n'=1}^{\ R\ \ \ N},\Big\{\bm{A}_j^{n',s}\Big\}_{s=1\ n'=1\ j=1}^{\ R\ \ \ N\ \ \ \ k_{n'}}\Big)$.
\bigskip

For $\Delta\in\mathbb{R}^{d_n\times d_n}$, using Taylor approximation we have
\begin{eqnarray*}
	\Phi_{i,n,r}(\bm{A}_i^{n,r}+\Delta)&=&\mathcal{L}\Big(\mathbf{\mathcal{W}}
	+\underset{n'=1}{\overset{n-1}\bigotimes} \prod_{j=1}^{k_{n'}}\bm{A}_j^{n',r}\bm{w}_r^{n'}\otimes \displaystyle\prod_{j=1}^{i-1}\bm{A}_j^{n,r}\Delta \prod_{j=i+1}^{k_n}\bm{A}_j^{n,r}\bm{w}_r^n\otimes \underset{n'=n+1}{\overset{N}\bigotimes}\displaystyle\prod_{j=1}^{k_{n'}}\bm{A}_j^{n',r}\bm{w}_r^{n'}
	\Big)\nonumber\\
	&=&\mathcal{L}({\mathcal{W}})+\left\langle \nabla \mathcal{L}({\mathcal{W}}),\underset{n'=1}{\overset{n-1}\bigotimes} \prod_{j=1}^{k_{n'}}\bm{A}_j^{n',r}\bm{w}_r^{n'}\otimes \displaystyle\prod_{j=1}^{i-1}\bm{A}_j^{n,r}\Delta \prod_{j=i+1}^{k_n}\bm{A}_j^{n,r}\bm{w}_r^n\otimes \underset{n'=n+1}{\overset{N}\bigotimes}\displaystyle\prod_{j=1}^{k_{n'}}
	\bm{A}_j^{n',r}\bm{w}_r^{n'}\right\rangle\nonumber\\
	&+&o(\|\Delta\|)\nonumber\\
	&=&\mathcal{L}(\mathbf{\mathcal{W}})+\left\langle [\nabla \mathcal{L}(\mathbf{\mathcal{W}})]_{(n)}\cdot \underset{n'\ne n}{\odot} \prod_{j=1}^{k_{n'}}\bm{A}_j^{n',r}\bm{w}_r^{n'},\prod_{j=1}^{i-1}\bm{A}_j^{n,r}\Delta \prod_{j=i+1}^{k_n}\bm{A}_j^{n,r}\bm{w}_r^n\right\rangle+o(\|\Delta\|)\nonumber\\
	&=&\mathcal{L}(\mathbf{\mathcal{W}})+\left\langle \Big (\prod_{j=1}^{i-1}\bm{A}_j^{n,r}\Big)^\top[\nabla \mathcal{L}(\mathbf{\mathcal{W}})]_{(n)}\Big (\underset{n'\ne n}{\odot}\prod_{j=1}^{k_{n'}}\bm{A}_j^{n',r}\bm{w}_r^{n'}\Big )\bm{w}_r^{n^{\top}}\Big (\prod_{j=i+1}^{k_n}\bm{A}_j^{n,r}\Big )^\top,\Delta\right\rangle\\
	&+&o(\|\Delta\|).\nonumber
\end{eqnarray*}
This implies that 
\begin{equation*}
	 \frac{\partial}{\partial \bm{A}_i^{n,r}}\Phi\Big(\!\{\bm{w}_{r'}^{n'}\}_{r'=1\ n'=1}^{\ R\ \ \ N},\!\Big\{\!\bm{A}_i^{n',r'}\!\Big\}_{r'=1\ n'=1\ i'=1}^{\ R\ \ \ N\ \ \ \ k_n'}\Big) \!=\!\Big(\!\displaystyle\prod_{j=1}^{i-1}\bm{A}_j^{n,r}\!\Big)^\top\![\nabla \mathcal{L}(\mathbf{\mathcal{W}})]_{(n)}\Big (\underset{n'\ne n}{\odot}\displaystyle\prod_{j=1}^{k_{n'}}\bm{A}_j^{n',r}\bm{w}_r^{n'}\Big ){\bm{w}_r^{n}}^\top\Big (\!\displaystyle\prod_{j=i+1}^{k_n}\bm{A}_j^{n,r}\Big )^\top.
\end{equation*}
Since $\displaystyle\frac{d}{dt}\bm{A}_i^{n,r}(t)=-\frac{\partial}{\partial \bm{A}_i^{n,r}}\Phi\Big (\{\bm{w}_{r'}^{n'}(t)\}_{r'=1\ n'=1}^{\ R\ \ \ N},\Big\{\bm{A}_{i'}^{n',r'}(t)\Big\}_{r'=1\ n'=1\ i'=1}^{\ R\ \ \ N\ \ \ \ k_{n'}}\Big)$, we have
\begin{equation}
\label{eq:deriveAi}
    \frac{d}{dt}\bm{A}_i^{n,r}(t)=- \Big(\displaystyle\prod_{j=1}^{i-1}\bm{A}_j^{n,r}(t)\Big)^\top[\nabla \mathcal{L}({\mathcal{W}(t)})]_{(n)}\Big (\underset{n'\ne n}{\odot}\displaystyle\prod_{j=1}^{k_{n'}}\bm{A}_j^{n',r}(t)\bm{w}_r^{n'}(t)\Big )\bm{w}_r^{n}(t)^\top\Big (\displaystyle\prod_{j=i+1}^{k_n}\bm{A}_j^{n,r}(t)\Big )^\top.
\end{equation}
Then
\begin{align}
	\frac{d}{dt}\|\bm{A}_i^{n,r}(t)\|^2&=2\left\langle \bm{A}_i^{n,r}(t),\frac{d}{dt}\bm{A}_i^{n,r}(t)\right\rangle\nonumber\\
	&=-2\left\langle \bm{A}_i^{n,r}(t),\frac{\partial}{\partial \bm{A}_i^{n,r}(t)}\Phi\Big (\{\bm{w}_{r'}^{n'}(t)\}_{r'=1\ n'=1}^{\ R\ \ \ N},\Big\{\bm{A}_{i'}^{n',r'}(t)\Big\}_{r'=1\ n'=1\ i'=1}^{\ R\ \ \ N\ \ \ \ k_{n'}}\Big)\right\rangle\nonumber\\
	&=-2\left\langle \!\bm{A}_i^{n,r}(t),\Big (\prod_{j=1}^{i-1}\bm{A}_j^{n,r}(t)\Big)^\top\![\nabla \mathcal{L}({\mathcal{W}(t)})]_{(n)}\Big (\!\underset{n'\ne n}{\odot}\prod_{j=1}^{k_{n'}}\bm{A}_j^{n',r}(t)\bm{w}_r^{n'}(t)\Big )\bm{w}_r^{n}(t)^\top\!\Big (\!\prod_{j=i+1}^{k_n}\bm{A}_j^{n,r}(t)\Big )^\top\right\rangle\nonumber\\
	&=-2\left\langle\prod_{j=1}^{i-1}\bm{A}_j^{n,r}(t)\cdot \bm{A}_i^{n,r}(t)\prod_{j=i+1}^{k_n}\bm{A}_j^{n,r}(t)\bm{w}_r^n(t),[\nabla \mathcal{L}({\mathcal{W}(t)})]_{(n)}\underset{n'\ne n}{\odot}\prod_{j=1}^{k_{n'}}\bm{A}_j^{n',r}(t)\bm{w}_r^{n'}(t)\right\rangle\nonumber\\
	&=-2\left\langle\prod_{j=1}^{k_n}\bm{A}_j^{n,r}(t)\bm{w}^{n,r}(t),[\nabla \mathcal{L}({\mathcal{W}(t)})]_{(n)}\underset{n'\ne n}{\odot}\prod_{j=1}^{k_{n'}}\bm{A}_j^{n',r}(t)\bm{w}_r^{n'}(t)\right\rangle.\nonumber\\
	&=-2\left\langle\nabla\mathcal{L}({\mathcal{W}(t)}),
	\underset{n'=1}{\overset{N}\bigotimes}\prod_{j=1}^{k_{n'}}\bm{A}_j^{n',r}(t)\bm{w}_r^{n'}(t)\right\rangle.  
	\nonumber
\end{align}
$\displaystyle\frac{d}{dt}\|\bm{A}_i^{n,r}(t)\|^2$ is independent of $n$ and $i$, then $\forall (n,m)\in \llbracket1,N\rrbracket^2$ $\forall (i,j)\in \llbracket1,k_n\rrbracket\times \llbracket1,k_m\rrbracket$ $\|\bm{A}_i^{n,r}(t)\|^2$ and $\|\bm{A}_j^{m,r}(t)\|^2$ have the same derivative, which implies that $\|\bm{A}_i^{n,r}(t)\|^2-\|\bm{A}_j^{m,r}(t)\|^2$ does not vary with time t. Thus 
$$\|\bm{A}_i^{n,r}(t)\|^2-\|\bm{A}_j^{m,r}(t)\|^2=\|\bm{A}_i^{n,r}(0)\|^2-\|\bm{A}_j^{m,r}(0)\|^2.$$

\subsubsection{Proof of $(ii)$}
We now compute $\displaystyle \frac{d}{dt}\|\bm{w}_r^n(t)\|^2$. We assume that $\{\bm{w}_{s}^{n'}\}_{(n',s)\neq(n,r)}$ and $\{\bm{A}_{j}^{n',s}\}_{s=1\ n'=1\ j=1}^{\ R\ \ \ N\ \ \ k_{n'}}$ are fixed, and consider $\Phi_r^n(\bm{w}_r^n)=\Phi\left(\{\bm{w}_{r'}^{n'}\}_{r'=1\ n'=1}^{\ R\ \ \ N}, \{\bm{A}_{i'}^{n',r'}\}_{r'=1\ n'=1\ i'=1}^{\ R\ \ \ \ N\ \ \ k_{n'}}\right)$.
\medskip

For $\Delta\in\mathbb{R}^{d_n}$, using Taylor approximation we have
\begin{eqnarray}
	\Phi_r^n(\bm{w}_r^n+\Delta)&=&\mathcal{L}\left(\mathbf{\mathcal{W}}+\underset{n'=1}{\overset{n-1}\bigotimes}\prod_{i=1}^{k_{n'}}\bm{A}_i^{n',r}\bm{w}_r^{n'}\otimes \prod_{i=1}^{k_n}\bm{A}_i^{n,r}\Delta \otimes \underset{n'=n+1}{\overset{N}\bigotimes}\prod_{i=1}^{k_{n'}}\bm{A}_i^{n',r}\bm{w}_r^{n'}\right)\nonumber\\
	&=&\mathcal{L}(\mathbf{\mathcal{W}})+\left\langle \nabla \mathcal{L}(\mathbf{\mathcal{W}}),\underset{n'=1}{\overset{n-1}\bigotimes}\prod_{i=1}^{k_{n'}}\bm{A}_i^{n',r}\bm{w}_r^{n'}\otimes \prod_{i=1}^{k_n}\bm{A}_i^{n,r}\Delta \otimes \underset{n'=n+1}{\overset{N}\bigotimes}\prod_{i=1}^{k_{n'}}\bm{A}_i^{n',r}\bm{w}_r^{n'}\right\rangle +o(\|\Delta \|) \nonumber\\
	&=&\mathcal{L}(\mathbf{\mathcal{W}})+\left\langle [\nabla \mathcal{L}(\mathbf{\mathcal{W}})]_{(n)}\cdot \Big (\underset{n'\ne n}{\odot}\prod_{i=1}^{k_{n'}}\bm{A}_i^{n',r}\bm{w}_r^{n'}\Big), \prod_{i=1}^{k_{n}}\bm{A}_i^{n,r}\Delta\right\rangle+o(\|\Delta \|)\nonumber\\
	&=&\mathcal{L}(\mathbf{\mathcal{W}})+\left\langle \Big (\prod_{i=1}^{k_{n'}}\bm{A}_i^{n,r}\Big)^\top[\nabla \mathcal{L}(\mathbf{\mathcal{W}})]_{(n)}\cdot \Big (\underset{n'\ne n}{\odot}\prod_{i=1}^{k_{n}}\bm{A}_i^{n',r}\bm{w}_r^{n'}\Big), \Delta\right\rangle+o(\|\Delta \|).\nonumber
\end{eqnarray}
This implies that $\displaystyle \frac{\partial}{\partial \bm{w}_r^n}\Phi\left(\{\bm{w}_{r'}^{n'}\}_{r'=1\ n'=1}^{\ R\ \ \ \ N}, \{A_{i'}^{n',r'}\}_{r'=1\ n'=1\ i'=1}^{R\ \ \ \ \ N\ \ \ \ \ \ k_{n'}}\right)=
\Big (\prod_{i=1}^{k_{n}}\bm{A}_i^{n,r}\Big )^\top[\nabla \mathcal{L}(\mathbf{\mathcal{W}})]_{(n)}\cdot \Big (\underset{n'\ne n}{\odot}\prod_{i=1}^{k_{n'}}\bm{A}_i^{n',r}\bm{w}_r^{n'}\Big)$. 
Then
\begin{eqnarray*}
	\frac{d}{dt}\|\bm{w}_r^n(t)\|^2&=&2\left\langle \bm{w}_r^n(t),\frac{d}{dt}\bm{w}_r^n(t)\right\rangle=-2\left\langle \bm{w}_r^n(t),\frac{\partial}{\partial \bm{w}_r^n}\Phi\left(\{\bm{w}_{r'}^{n'}(t)\}_{r'=1\ n'=1}^{\ R\ \ \ \ N}, \{A_{i'}^{n',r'}(t)\}_{r'=1\ n'=1\ i'=1}^{\ R\ \ \ N\ \ \ \ k_{n'}}\right)\right\rangle\\
	&=&-2\left\langle \bm{w}_r^n(t),\Big (\prod_{i=1}^{k_{n}}\bm{A}_i^{n,r}(t)\Big )^\top[\nabla \mathcal{L}(\mathcal{W}(t))]_{(n)}\cdot \Big (\underset{n'\ne n}{\odot}\prod_{i=1}^{k_{n'}}\bm{A}_i^{n',r}(t)\bm{w}_r^{n'}(t)\Big )\right\rangle\\
	&=&-2\left\langle\prod_{i=1}^{k_{n}}\bm{A}_{i}^{n,r}(t).\bm{w}_{r}^{n}(t),[\nabla \mathcal{L}(\mathcal{W}(t))]_{(n)}\cdot\Big(\underset{n'\ne n}{\odot}\prod_{i=1}^{k_{n'}}\bm{A}_{i}^{n',r}(t)\bm{w}_r^{n'}(t)\Big)\right\rangle\\
	&=&2\left\langle-\nabla \mathcal{L}(\mathcal{W}(t)),\underset{n'=1}{\overset{N}\bigotimes}
	\prod_{i=1}^{k_{n'}}\bm{A}_i^{n',r}(t)\bm{w}_r^{n'}(t)\right\rangle
\end{eqnarray*}
$\displaystyle\frac{d}{dt}\|\bm{w}_r^n(t)\|^2$ is independent of $n$ then $\forall n, m\in \llbracket1,N\rrbracket$, $\|\bm{w}_r^n(t)\|^2$ and $\|\bm{w}_r^m(t)\|^2$ have the same derivative, which implies that
$ \|\bm{w}_r^n(t)\|^2-\|\bm{w}_r^m(t)\|^2$ dos not vary with time t. Thus
$$\|\bm{w}_r^n(t)\|^2-\|\bm{w}_r^m(t)\|^2=\|\bm{w}_r^n(0)\|^2-\|\bm{w}_r^m(0)\|^2.$$

\subsubsection{Proof of $(iii)$}
From the proofs of $(i)$ and $(ii)$, we have, $\forall n,m\in \llbracket1,N\rrbracket$ and $i\in \llbracket1,k_n\rrbracket$,
$$\frac{d}{dt}\|\bm{w}_r^n(t)\|^2=\frac{d}{dt}\|\bm{A}_i^{n,r}(t)\|^2 = 2\left\langle-\nabla \mathcal{L}(\mathcal{W}(t)),\underset{n'=1}{\overset{N}\bigotimes}
\prod_{i=1}^{k_{n'}}\bm{A}_i^{n',r}(t)\bm{w}_r^{n'}(t)\right\rangle,$$ which implies that $$ \|\bm{w}_r^n(t)\|^2-\|\bm{A}_i^{n,r}(t)\|^2=\|\bm{w}_r^n(0)\|^2-\|\bm{A}_i^{n,r}(0)\|^2.$$

\subsection{Proof of Theorem~\ref{th:dy_deepcp}}
\subsubsection{Proof of $(i)$}
First note that \begin{eqnarray}\frac{d}{dt}\|\bm{w}_r^n(t)\|&=&\frac{1}{2\|\bm{w}_r^n(t)\|}\frac{d}{dt}\Big (\|\bm{w}_r^n(t)\|^2\Big )\nonumber\\
	&=&\frac{1}{\|\bm{w}_r^n(t)\|}\left\langle -\nabla \mathcal{L}(\mathcal{W}(t)),\underset{n'=1}{\overset{N}\bigotimes}
	\prod_{i=1}^{k_{n'}}\bm{A}_i^{n',r}(t)\bm{w}_r^{n'}(t)\right\rangle.\nonumber
\end{eqnarray}
We now compute $\displaystyle\frac{d}{dt}\left\|\underset{n=1}{\overset{N}\bigotimes}\bm{w}_r^n(t)\right\|$.
\begin{eqnarray}
	&&\frac{d}{dt}\left\|\underset{n=1}{\overset{N}\bigotimes}\bm{w}_r^n(t)\right\|=\frac{d}{dt}\Big (\prod_{n=1}^N\|\bm{w}_r^n(t)\|\Big )\nonumber\\
	&&=\sum_{n=1}^N\frac{d}{dt}\|\bm{w}_r^n(t)\|\prod_{n'\ne n}\|\bm{w}_r^{n'}(t)\|\nonumber\\
	&&=\sum_{n=1}^N\frac{1}{\|\bm{w}_r^n(t)\|}\left\langle-\nabla\mathcal{L}(\mathcal{W}(t)),
	\underset{n'=1}{\overset{N}\bigotimes}\prod_{i=1}^{k_{n'}}\bm{A}_i^{n',r}(t)\bm{w}_r^{n'}(t)\right\rangle
	\prod_{n'\ne n}\|\bm{w}_r^{n'}(t)\|\nonumber\\
	&&=\sum_{n=1}^N\left\langle-\nabla\mathcal{L}(\mathcal{W}(t)),\underset{n'=1}{\overset{N}\bigotimes}
	\prod_{i=1}^{k_{n'}}\bm{A}_i^{n',r}(t)\widehat{\bm{w}}_r^{n'}(t)\right\rangle\prod_{n'\ne n}\|\bm{w}_r^{n'}(t)\|^2, \nonumber
\end{eqnarray}
where $\widehat{\bm{w}}_r^{n'}(t)=\frac{\bm{w}_r^{n'}(t)}{\|\bm{w}_r^{n'}(t)\|}$.

Let $\gamma_r(t):=\left\langle -\nabla\mathcal{L}(\mathcal{{W}}(t)),\displaystyle\underset{n'=1}{\overset{N}\bigotimes}\displaystyle\prod_{i=1}^{k_{n'}}
\bm{A}_i^{n',r}(t)\widehat{\bm{w}}_r^{n'}(t)\right\rangle$
and assume that $\gamma_r(t)\geq 0$, we have 
$$\|\bm{w}_r^n(t)\|^2\leq \min_{n'\in \llbracket1,N\rrbracket}\|\bm{w}_r^{n'}(t)\|^2+\|\bm{w}_r^n(t)\|^2-\min_{n'\in \llbracket1,N\rrbracket}\|\bm{w}_r^{n'}(t)\|^2$$
So, $$\|\bm{w}_r^n(t)\|^2\leq  \min_{n'\in \llbracket1,N\rrbracket}\|\bm{w}_r^{n'}(t)\|^2+\varepsilon_1(t)= \Big (\Big (\min_{n'\in \llbracket1,N\rrbracket}\|\bm{w}_r^{n'}(t)\|\Big)^N\Big)^{\frac{2}{N}}+\varepsilon_1(t),$$

where $\varepsilon_1(t):=\displaystyle\max\limits_{\underset{(n,m)\in
		\llbracket1,N\rrbracket^2}{r\in\{1\hdots,R\}}}\Big |\|\bm{w}_r^n(t)\|^2-\|\bm{w}_r^m(t)\|^2\Big|
	$
	denotes the unbalancedness magnitude of the weight vectors $\bm{w}_r^n(t)$.

$\Longrightarrow \|\bm{w}_r^n(t)\|^2\leq \Big (\displaystyle\prod_{n'=1}^N\|\bm{w}_r^{n'}(t)\|\Big)^{\frac{2}{N}}+\varepsilon_1(t)
=\left\|\underset{n=1}{\overset{N}\bigotimes}\bm{w}_r^n(t)\right\|^{\frac{2}{N}}+\varepsilon_1(t)$

$\overset{\gamma_r(t)\geq 0}{\Longrightarrow}\displaystyle\prod_{n'\ne n}\|\bm{w}_r^{n'}(t)\|^2\gamma_r(t)\leq \Big (\left\|\underset{n=1}{\overset{N}\bigotimes}\bm{w}_r^n(t)\right\|^{\frac{2}{N}}+\varepsilon_1(t)\Big )^{N-1}\gamma_r(t).$
This gives the following upper bound:
$$\displaystyle \frac{d}{dt}\left\|\underset{n=1}{\overset{N}\bigotimes}\bm{w}_r^n(t)\right\|\leq N\gamma_r(t)\Big (\left\|\underset{n=1}{\overset{N}\bigotimes}\bm{w}_r^n(t)\right\|^{\frac{2}{N}}+\varepsilon_1(t)\Big )^{N-1}.$$

On the other hand, we can write
$$\frac{d}{dt}\left\|\underset{n=1}{\overset{N}\bigotimes}\bm{w}_r^n(t)\right\|
=\left\|\underset{n=1}{\overset{N}\bigotimes}\bm{w}_r^n(t)\right\|^2
\sum_{n=1}^N\frac{1}{\|\bm{w}_r^n(t)\|^2}\gamma_r(t).$$
Since, $$\frac{1}{\|\bm{w}_r^n(t)\|^2}\geq \frac{1}{\|\underset{n=1}{\overset{N}\bigotimes}\bm{w}_r^n(t)\|^{\frac{2}{N}}+\varepsilon_1(t)},$$
We obtain the following lower bound:
$$\frac{d}{dt}\left\|\underset{n=1}{\overset{N}\bigotimes}\bm{w}_r^n(t)\right\|\geq N\gamma_r(t)\frac{\left\|\underset{n=1}{\overset{N}\bigotimes}\bm{w}_r^n(t)\right\|^2}
{\left\|\underset{n=1}{\overset{N}\bigotimes}\bm{w}_r^n(t)\right\|^{\frac{2}{N}}+\varepsilon_1(t)}.$$
Thus
$$N\gamma_r(t)\frac{\left\|\underset{n=1}{\overset{N}\bigotimes}\bm{w}_r^n(t)\right\|^2}
{\left\|\underset{n=1}{\overset{N}\bigotimes}\bm{w}_r^n(t)\right\|^{\frac{2}{N}}+\varepsilon_1(t)}\leq \frac{d}{dt}\left\|\underset{n=1}{\overset{N}\bigotimes}\bm{w}_r^n(t)\right\| \leq N\gamma_r(t)\Big(\left\|\underset{n=1}{\overset{N}\bigotimes}\bm{w}_r^n(t)\right\|^{\frac{2}{N}}+\varepsilon_1(t)\Big)^{N-1}.$$
When $\gamma_r(t)\leq 0$, this inequality is reversed.

It is easy to see that when $\varepsilon(0) = 0$, $\varepsilon_1(0)$ will be also equal to zero. Moreover by Lemma~\ref{lem:dy}, $\varepsilon_1(t)$ stays constant over time, and thus
\begin{align*}
 \frac{d}{dt}\left\|\underset{n=1}{\overset{N}\bigotimes}\bm{w}_r^n(t)\right\|
&=N\left\|\underset{n=1}{\overset{N}\bigotimes}\bm{w}_r^n(t)\right\|^{2-\frac{2}{N}}\left\langle-\nabla \mathcal{L}(\mathbf{\mathcal{W}}(t)),\underset{n'=1}{\overset{N}\bigotimes}
\prod_{i=1}^{k_{n'}}\bm{A}_i^{n',r}(t)\widehat{\bm{w}}_r^{n'}(t)\right\rangle.\\
&=N\left\|\underset{n=1}{\overset{N}\bigotimes}\bm{w}_r^n(t)\right\|^{2-\frac{2}{N}}\prod_{n'=1}^N
\prod_{i=1}^{k_{n'}}\|\bm{A}_i^{n',r}(t)\|\left\langle-\nabla \mathcal{L}({\mathcal{W}(t)}),\underset{n'=1}{\overset{N}\bigotimes}
\prod_{i=1}^{k_{n'}}\widehat{\bm{A}}_i^{n',r}(t)\widehat{\bm{w}}_r^{n'}(t)\right\rangle,
\end{align*}
where $\widehat{\bm{A}}_i^{n',r}(t)=\displaystyle\frac{\bm{A}_i^{n',r}(t)}{\|\bm{A}_i^{n',r}(t)\|}$.

Assuming that $\varepsilon(0)=0$, which implies that $\varepsilon_1(0)=0$, and using Lemma~\ref{lem:dy} we also obtain that 
$$\|\bm{A}_i^{n,r}(t)\|=\|\bm{w}_r^m(t)\|=\left\|\underset{m=1}{\overset{N}\bigotimes}\bm{w}_r^m(t)\right\|^{\frac{1}{N}}.$$
Plugging this into the equality just above gives
\begin{align*}
\frac{d}{dt}\|\underset{n=1}{\overset{N}\bigotimes}\bm{w}_r^n(t)\|&=N\left\|\underset{n=1}{\overset{N}\bigotimes}
\bm{w}_r^n(t)\right\|^{2-\frac{2}{N}}\prod_{n'=1}^N
\prod_{i=1}^{k_{n'}}\left\|\underset{n=1}{\overset{N}\bigotimes}\bm{w}_r^n(t)\right\|^{\frac{1}{N}}
\left\langle-\nabla \mathcal{L}({\mathcal{W}(t)}),\underset{n'=1}{\overset{N}\bigotimes}\prod_{i=1}^{k_{n'}}
\widehat{\bm{A}}_i^{n',r}(t)\widehat{\bm{w}}_r^{n'}(t)\right\rangle \\
&=N\left\|\underset{n=1}{\overset{N}\bigotimes}\bm{w}_r^n(t)\right\|^{2-\frac{2}{N}+\frac{k_1+\hdots+k_N}{N}}\left\langle-\nabla \mathcal{L}({\mathcal{W}(t)}),\underset{n'=1}{\overset{N}\bigotimes}\prod_{i=1}^{k_{n'}}
\widehat{\bm{A}}_i^{n',r}(t)\widehat{\bm{w}}_r^{n'}(t)\right\rangle,
\end{align*}
which completes the proof.


\subsubsection{Proof of $(ii)$}

The proof is based on the following lemma which characterizes the evolution of the singular values of the matrix $\bm{A}^{n,r}(t) :=\displaystyle\prod_{i=1}^{k_{n}}\bm{A}_i^{n,r}(t)$. We denote by $\left( \sigma_{l}^{n,r} (t) \right)_{1\leq l\leq d_n}$ the singular values of $\bm{A}^{n,r}(t)$. 
The singular value decomposition of $\bm{A}^{n,r}(t)$ is $\bm{A}^{n,r}(t) = \bm{U}^{n,r}(t) \bm{S}^{n,r}(t) \bm{V}^{n,r}(t)^\top$, with $\bm{S}^{n,r}(t) = \mathrm{diag} \{ \sigma_{l}^{n,r} (t), 1\leq l\leq d_n\}$ and $\bm{U}^{n,r}(t) \in \mathbb{R}^{d_n \times d_n}$ and $\bm{V}^{n,r}(t)\in \mathbb{R}^{d_n \times d_n}$ are two orthogonal matrices.

\medskip
\begin{lemma}
If  
$\{\bm{A}_i^{n,r}(0)\}_{i=1}^{k_n}$
are matrices satisfying $\bm{A}_i^{n,r}(0)^\top \bm{A}_i^{n,r}(0) = \bm{A}_{i+1}^{n,r}(0) \bm{A}_{i+1}^{n,r}(0)^\top,$ for all  $i\in\{1,\ldots,k_n-1\}$,  then the singular values of the product matrix $\bm{A}^{n,r}(t)$ evolve by:
\begin{equation*}
\frac{d}{dt} \big( \sigma_{l}^{n,r} (t) \big) = k_n \big( \sigma_{l}^{n,r} (t) \big)^{2(1-\frac{1}{k_n})} \beta_l^{n,r}(t)\quad , \ l=1,\ldots, d_n,
\end{equation*}
where $\beta_l^{n,r}(t) = \left\langle -\nabla \mathcal{L}(\bm{\mathcal{W}}(t)), \underset{n'=1}{\overset{n-1}\bigotimes} \prod_{i=1}^{k_{n'}}\bm{A}_i^{n',r}(t)\bm{w}_r^{n'}(t)\otimes \bm{u}_l^{n,r}(t) \bm{v}_l^{n,r}(t)^\top \bm{w}_r^{n}(t) \otimes \underset{n'=n+1}{\overset{N}\bigotimes}\displaystyle\prod_{i=1}^{k_{n'}}
	\bm{A}_i^{n',r}(t)\bm{w}_r^{n'}(t)
\right\rangle$, and $\bm{u}_l^{n,r}(t)$ and $\bm{v}_l^{n,r}(t)$ are the $l$-th columns of $\bm{U}^{n,r}(t)$ and $\bm{V}^{n,r}(t)$, respectively.
\label{lem:sv}
\end{lemma}
\begin{proof}
First, we use the same arguments of the proof of Theorem~1 in \citet{arora2018optimization} to prove that:
$$\bm{A}_i^{n,r}(t)^\top \bm{A}_i^{n,r}(t) = \bm{A}_{i+1}^{n,r}(t) \bm{A}_{i+1}^{n,r}(t)^\top \quad, \ \forall t \geq 0, \forall i\in\{1,\ldots,k_n-1\}.$$
Indeed, since by~\eqref{eq:deriveAi} we have
\begin{equation*}
    \frac{d}{dt} \bm{A}_i^{n,r}(t) = - \Big(\displaystyle\prod_{j=1}^{i-1}\bm{A}_j^{n,r}(t)\Big)^\top[\nabla \mathcal{L}(\mathbf{\mathcal{W}})]_{(n)}\Big (\underset{n'\ne n}{\odot}\displaystyle\prod_{j=1}^{k_{n'}}\bm{A}_j^{n',r}(t)\bm{w}_r^{n'}(t)\Big )\bm{w}_r^{n}(t)^\top\Big (\displaystyle\prod_{j=i+1}^{k_n}\bm{A}_j^{n,r}(t)\Big )^\top,
\end{equation*}
then $\displaystyle\frac{d}{dt} \left(\bm{A}_i^{n,r}(t)^\top \bm{A}_i^{n,r}(t)\right) = \frac{d}{dt} \left( \bm{A}_{i+1}^{n,r}(t) \bm{A}_{i+1}^{n,r}(t)^\top \right)$. Having that $\bm{A}_i^{n,r}(0)^\top \bm{A}_i^{n,r}(0) = \bm{A}_{i+1}^{n,r}(0) \bm{A}_{i+1}^{n,r}(0)^\top$, we obtain \begin{equation}
\label{eq:AiAi+1}
    \bm{A}_i^{n,r}(t)^\top \bm{A}_i^{n,r}(t) = \bm{A}_{i+1}^{n,r}(t) \bm{A}_{i+1}^{n,r}(t)^\top, \forall t \geq 0. 
\end{equation}
Using the singular value decomposition of $\bm{A}_{i}^{n,r}(t)$ and $\bm{A}_{i+1}^{n,r}(t)$, \eqref{eq:AiAi+1} implies that   $\bm{A}_{i}^{n,r}(t)$ and $\bm{A}_{i+1}^{n,r}(t)$ have the same singular values. So, the two matrices can then be written as $\bm{A}_{i}^{n,r}(t) = \bm{U}_i^{n,r}(t) \bm{\Sigma}^{n,r}(t) \bm{V}_i^{n,r}(t)^\top$ and $\bm{A}_{i+1}^{n,r}(t) = \bm{U}_{i+1}^{n,r}(t) \bm{\Sigma}^{n,r}(t) \bm{V}_{i+1}^{n,r}(t)^\top$, where $\bm{\Sigma}^{n,r}(t) = \mathrm{diag}\big(\lambda_{1}^{n,r}(t) \bm{I}_{\alpha_1}, \ldots, \lambda_{m}^{n,r}(t) \bm{I}_{\alpha_m} \big)$, where, $\forall s\in\{1,\ldots,m\}$, $\alpha_s$ is the multiplicity of the singular value $\lambda_{s}^{n,r}(t)$ and $\bm{I}_{\alpha_s}$ is the $\alpha_s \times \alpha_s$ identity matrix. 
Moreover, \eqref{eq:AiAi+1} also implies that
$$ \bm{U}_{i+1}^{n,r}(t) = \bm{V}_{i}^{n,r}(t) \bm{O}_{i}^{n,r}(t), $$ where $\bm{O}_{i}^{n,r}(t) = \mathrm{diag}\big( \bm{O}_{i,1}^{n,r}(t),\ldots, \bm{O}_{i,m}^{n,r}(t) \big)$ and $\bm{O}_{i,s}^{n,r}(t) \in \mathbb{R}^{\alpha_s \times \alpha_s}$ is an orthogonal matrix, $\forall s \in\{1,\ldots,m\}$~\citep[Section A.1]{arora2018optimization}.

Using the fact that, $\forall i \in\{1,\ldots,k_n-1\}$, $\bm{O}_{i}^{n,r}(t)$ and $\bm{\Sigma}^{n,r}(t)$ commute, we obtain that 
\begin{equation}
\label{eq:svdA}
    \bm{A}^{n,r}(t) :=\displaystyle\prod_{i=1}^{k_{n}}\bm{A}_i^{n,r}(t) = \bm{U}_1^{n,r}(t) \prod_{i=1}^{k_n-1}\bm{O}_i^{n,r}(t) \big(\bm{\Sigma}^{n,r}(t)\big)^{k_n} \bm{V}_{k_n}^{n,r}(t)^\top,
\end{equation}
and 
\begin{equation}
\label{eq:s=sigma}
    \bm{S}^{n,r}(t) = \big(\bm{\Sigma}^{n,r}(t)\big)^{k_n}.
\end{equation}

We now characterize the evolution of the singular values of $\bm{A}^{n,r}(t)$ using the same arguments as in the proof of Theorem 3 in~\citet{arora2019implicit}. 
\begin{align}
\label{eq:deriveA}
    & \frac{d}{dt} \bm{A}^{n,r}(t) = \sum_{j=1}^{k_n} \prod_{i=1}^{j-1} \bm{A}_{i}^{n,r}(t) \big( \frac{d}{dt} \bm{A}_{j}^{n,r}(t)\big) \prod_{i=j+1}^{k_n} \bm{A}_{i}^{n,r}(t) \nonumber \\
    & = - \sum_{j=1}^{k_n} \prod_{i=1}^{j-1} \bm{A}_{i}^{n,r}(t)
    \Big(\displaystyle\prod_{i=1}^{j-1}\bm{A}_i^{n,r}(t)\Big)^\top[\nabla \mathcal{L}(\mathbf{\mathcal{W}})]_{(n)}\Big (\underset{n'\ne n}{\odot}\displaystyle\prod_{i=1}^{k_{n'}}\bm{A}_i^{n',r}(t)\bm{w}_r^{n'}(t)\Big )\bm{w}_r^{n}(t)^\top\Big (\displaystyle\prod_{i=j+1}^{k_n}\bm{A}_i^{n,r}(t)\Big )^\top
    \prod_{i=j+1}^{k_n} \bm{A}_{i}^{n,r}(t) \nonumber \\
    & \overset{(*)}{=} - \sum_{j=1}^{k_n} \left[ \bm{A}^{n,r}(t) \bm{A}^{n,r}(t)^\top\right]^{\frac{j-1}{k_n}} 
    [\nabla \mathcal{L}(\mathbf{\mathcal{W}})]_{(n)}\Big (\underset{n'\ne n}{\odot}\displaystyle\prod_{i=1}^{k_{n'}}\bm{A}_i^{n',r}(t)\bm{w}_r^{n'}(t)\Big )\bm{w}_r^{n}(t)^\top
    \left[ \bm{A}^{n,r}(t)^\top \bm{A}^{n,r}(t)\right]^{\frac{k_n-j}{k_n}}.
    \end{align}
(*) follows from~\eqref{eq:svdA} as in~\citet[Proof of Thm. 1]{arora2018optimization}.

Let $\bm{u}_l^{n,r}(t)$ and $\bm{v}_l^{n,r}(t)$ be the $l$-th columns of $\bm{U}^{n,r}(t)$ and $\bm{V}^{n,r}(t)$, respectively.
Using Eq. 25 in \citet{arora2019implicit}, we have
\begin{align*}
    & \frac{d}{dt} \sigma_l^{n,r}(t)  = \bm{u}_{l}^{n,r}(t)^\top \left[\frac{d}{dt} \bm{A}^{n,r}(t)\right] \bm{v}_{l}^{n,r}(t) \\
    & = -k_n \left( \sigma_l^{n,r}(t)\right)^{2\frac{k_n-1}{k_n}} \bm{u}_{l}^{n,r}(t)^\top [\nabla \mathcal{L}(\mathbf{\mathcal{W}})]_{(n)}\Big (\underset{n'\ne n}{\odot}\displaystyle\prod_{i=1}^{k_{n'}}\bm{A}_i^{n',r}(t)\bm{w}_r^{n'}(t)\Big )\bm{w}_r^{n}(t)^\top
    \bm{v}_{l}^{n,r}(t) \text{ \ (using Eq.~\ref{eq:deriveA} and Eq.~\ref{eq:svdA})}\\
    & =  -k_n \left( \sigma_l^{n,r}(t)\right)^{2(1-\frac{1}{k_n})} 
\left\langle \bm{u}_{l}^{n,r}(t) \bm{v}_{l}^{n,r}(t)^\top \bm{w}_r^{n}(t),  [\nabla \mathcal{L}(\mathbf{\mathcal{W}})]_{(n)}\Big (\underset{n'\ne n}{\odot}\displaystyle\prod_{i=1}^{k_{n'}}\bm{A}_i^{n',r}(t)\bm{w}_r^{n'}(t)\Big) \right\rangle   \\
& \overset{(**)}{=}  -k_n \left( \sigma_l^{n,r}(t)\right)^{2(1-\frac{1}{k_n})} 
\left\langle \nabla \mathcal{L}(\mathbf{\mathcal{W}}), \underset{n'=1}{\overset{n-1}\bigotimes} \prod_{i=1}^{k_{n'}}\bm{A}_i^{n',r}(t)\bm{w}_r^{n'}(t)\otimes \bm{u}_l^{n,r}(t) \bm{v}_l^{n,r}(t)^\top \bm{w}_r^{n}(t) \otimes \underset{n'=n+1}{\overset{N}\bigotimes}\displaystyle\prod_{i=1}^{k_{n'}}
	\bm{A}_i^{n',r}(t)\bm{w}_r^{n'}(t) \right\rangle.
\end{align*}
(**) follows from Lemma~\ref{lem1}.
\end{proof}

\paragraph{{Proof of Theorem~\ref{th:dy_deepcp} $\mathbf{(ii)}$}}
 As shown in Lemma~\ref{lem:sv}, if $\bm{A}_i^{n,r}(0)^\top \bm{A}_i^{n,r}(0) = \bm{A}_{i+1}^{n,r}(0) \bm{A}_{i+1}^{n,r}(0)^\top$, then all the matrices $\{\bm{A}_{i}^{n,r}(t)\}_{i=1}^{k_n}$ have the same singular values, $\forall t \geq 0$.
Let $\bm{\Sigma}^{n,r}(t) := \mathrm{diag} \{ \rho_{l}^{n,r} (t), 1\leq l\leq d_n\}$ be their diagonal singular value  matrix. 
Using~\eqref{eq:s=sigma}, we then have
\begin{equation}
\label{eq:sv}
\sigma_l^{n,r}(t) = \big(\rho_{l}^{n,r} (t)\big)^{k_n},
\end{equation}
where $\{\sigma_l^{n,r} (t)\}_{l=1}^{d_n}$ are the singular values of $\bm{A}^{n,r}(t):=\displaystyle\prod_{i=1}^{k_{n}}\bm{A}_i^{n,r}(t)$.

Using Lemma~\ref{lem:sv} above and Lemma~4 in~\citet{arora2019implicit} and since $1-\frac{1}{k_n} \geq \frac{1}{2}$, we obtain that $\sigma_l^{n,r}(t) = 0$, $\forall t\geq 0$, if $\sigma_l^{n,r}(0) = 0$. Thus by~\eqref{eq:sv}, we  have $\rho_l^{n,r}(t) = 0$, $\forall t\geq 0$, if $\rho_l^{n,r}(0) = 0$. 
Since $\{\bm{A}_i^{n,r}(0)\}_{i=1}^{k_n}$ are rank-one matrices, then $\rho_l^{n,r}(0) = 0$, $\forall l \in \{2,\ldots,d_n\}$. This implies that $\rho_l^{n,r}(t) = 0$, $\forall l \in \{2,\ldots,d_n\}, \forall t\geq 0$.
\begin{align*}
    \left\|\underset{i=1}{\overset{k_n}\prod}{\bm{A}}_i^{n,r}(t){\bm{w}}_r^{n}(t)\right\|^2 & = \left\| \bm{U}_1^{n,r}(t) \prod_{i=1}^{k_n-1}\bm{O}_i^{n,r}(t) \big(\bm{\Sigma}^{n,r}(t)\big)^{k_n} \bm{V}_{k_n}^{n,r}(t)^\top \bm{w}_r^{n}(t)\right\|^2    \text{(by Eq.~\ref{eq:svdA})} \\
    %
    %
    %
    & = \left\| \big(\bm{\Sigma}^{n,r}(t)\big)^{k_n} \bm{V}^{n,r}_{k_n}(t)^\top \bm{w}_r^{n}(t)\right\|^2\\
    & = \big(\rho_1^{n,r}(t)\big)^{2k_n} \left\langle \tilde{\bm{v}}_{r}^{n}(t), \bm{w}_r^{n}(t)\right\rangle^2,
\end{align*}
with $\tilde{\bm{v}}_{r}^{n}(t)$ is the first column of $\bm{V}_{k_n}^{n,r}(t)$.

Now, we have
\begin{align*}
    \delta_r(t) &:=  \left\langle -\nabla \mathcal{L}(\mathcal{W}), \underset{n=1}{\overset{N}\bigotimes} 
\underset{i=1}{\overset{k_n}\prod}
\widehat{\bm{A}}_i^{n,r}(t)\widehat{\bm{w}}_r^{n}(t)\right\rangle\\
& = \left\langle -\nabla \mathcal{L}(\mathcal{W}), \underset{n=1}{\overset{N}\bigotimes} \frac{
\underset{i=1}{\overset{k_n}\prod}
{\bm{A}}_i^{n,r}(t){\bm{w}}_r^{n}(t)}{\left(\underset{i=1}{\overset{k_n}\prod}\left\| {\bm{A}}_i^{n,r}(t) \right\|\right) \| \bm{w}_r^{n}(t)\|}\right\rangle\\
& = \left\langle -\nabla \mathcal{L}(\mathcal{W}), \underset{n=1}{\overset{N}\bigotimes} \frac{
\underset{i=1}{\overset{k_n}\prod}
{\bm{A}}_i^{n,r}(t){\bm{w}}_r^{n}(t)} {\left\|\underset{i=1}{\overset{k_n}\prod}{\bm{A}}_i^{n,r}(t){\bm{w}}_r^{n}(t)\right\|}\right\rangle \underset{n=1}{\overset{N}\prod}\frac{\left\|\underset{i=1}{\overset{k_n}\prod}{\bm{A}}_i^{n,r}(t){\bm{w}}_r^{n}(t)\right\|}{ \left(\underset{i=1}{\overset{k_n}\prod}\left\| {\bm{A}}_i^{n,r}(t) \right\|\right) \| \bm{w}_r^{n}(t)\|}.
\end{align*}
Moreover,
\begin{align*}
    \underset{n=1}{\overset{N}\prod}\frac{\left\|\underset{i=1}{\overset{k_n}\prod}{\bm{A}}_i^{n,r}(t){\bm{w}}_r^{n}(t)\right\|}{ \left(\underset{i=1}{\overset{k_n}\prod}\left\| {\bm{A}}_i^{n,r}(t) \right\|\right) \| \bm{w}_r^{n}(t)\|} 
    &= 
    \underset{n=1}{\overset{N}\prod}\frac{\big(\rho_1^{n,r}(t)\big)^{k_n} |\left\langle \tilde{\bm{v}}_r^{n}(t), \bm{w}_r^{n}(t) \right\rangle |}{\big(\rho_1^{n,r}(t)\big)^{k_n} \| \bm{w}_r^{n}(t) \|}\\
    & = \underset{n=1}{\overset{N}\prod} |\left\langle \tilde{\bm{v}}_r^{n}(t), \widehat{\bm{w}}_r^{n}(t) \right\rangle |\\
    & = \left|  \left\langle \underset{n=1}{\overset{N}\bigotimes}{\tilde{\bm{v}}}_r^{n}(t) , \underset{n=1}{\overset{N}\bigotimes} \widehat{\bm{w}}_r^{n}(t) \right\rangle \right|,
\end{align*}
which concludes the proof.


\section{An Experiment with Rank-One Matrix Initialization}
\label{app:xp}

 \begin{figure}[hbt!]
 	\centering
	\hspace*{-0.3cm}\includegraphics[scale=0.54]{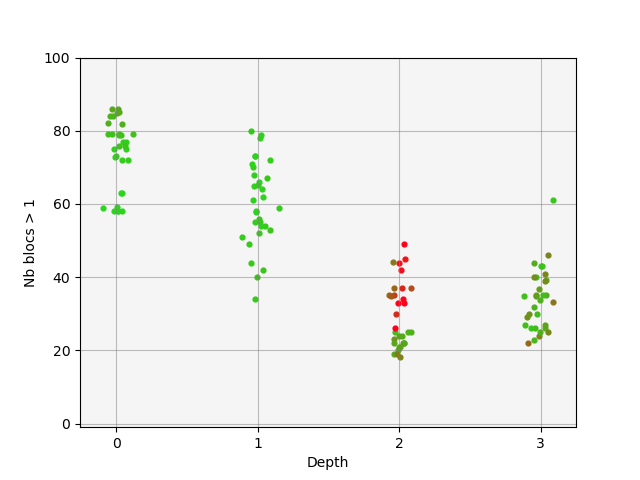}
	\caption{Tensor completion using Meteo-UK data sets with rank-one initialization of the matrix parameters: analysis of the impact of the depth on the rank of the learned tensor. The figure shows for shallow (on the left, depth=0) to deep models (up to depth=5, on the right) the effective rank of the learned tensors for a number of runs that differ from initialization setting. Every single point stands for a learning experiment. Points are plotted with a small random displacement in x and y coordinates to better see point clouds. The color represents the test loss of the model, from green to red respectively for small to large loss}
 		\label{fig:r1}
\end{figure}

We conduct the same experiment as in Section~\ref{subsec:realdata} but with matrices $\bm{A}_i^{n,r}$ initialized by random rank-one matrices.
We used Meteo-UK data set and launched more than 125 learning experiments using various initialization parameters and seeds, for depth ranging from 0 to~5. We report for each experiment the effective CP-rank of the model, i.e. the number of blocks that emerged at convergence~(blocks that have a norm greater than $1$). The percentage of observed and unobserved inputs in the tensor are 20\% and 80\%, respectively.
Figure~\ref{fig:r1} shows the same behaviour as Figures~\ref{fig:multiruns} and~\ref{fig:realdata}. For shallower architectures, the impact of initialization is huge and the solutions are mostly of high rank. For deeper architectures, the learned
tensor generally has a lower rank.


\end{document}